\documentclass[11pt]{article}
\usepackage[a4paper,top=3cm,bottom=2cm,left=2cm,right=2cm,marginparwidth=1.75cm]{geometry}
\usepackage{fullpage}
\usepackage{graphics} 
\usepackage{epsfig}
\usepackage{amsmath} 
\usepackage{amssymb} 
\usepackage{amsthm}
\usepackage{tabularx,booktabs}
\usepackage{mathrsfs}
\usepackage{selectp}
\usepackage{Shorthands}
\usepackage[normalem]{ulem}
\usepackage{authblk}
\usepackage{subcaption}
\usepackage{natbib}
\usepackage{Shorthands}
\usepackage{algorithm}
\usepackage{algpseudocode}
\usepackage{wrapfig}
\usepackage{scalerel,stackengine}
\stackMath
\newcommand\reallywidehat[1]{%
\savestack{\tmpbox}{\stretchto{%
  \scaleto{%
    \scalerel*[\widthof{\ensuremath{#1}}]{\kern-.6pt\bigwedge\kern-.6pt}%
    {\rule[-\textheight/2]{1ex}{\textheight}}
  }{\textheight}%
}{0.5ex}}%
\stackon[1pt]{#1}{\tmpbox}%
}

\usepackage{xr}
\usepackage{xcolor}

\newboolean{showcomments}
\setboolean{showcomments}{true}
\ifthenelse{\boolean{showcomments}}
{ \newcommand{\mynote}[3]{
		\fbox{\bfseries\sffamily\scriptsize#1}
		{\small$\blacktriangleright$\textsf{\emph{\color{#3}{#2}}}$\blacktriangleleft$}}
	\newcommand{\zzz}[1]{{\setlength{\fboxsep}{2pt}\fcolorbox{black}{yellow}{\textsf{\emph{#1}}}}\xspace}}
{ \newcommand{\mynote}[3]{}
	\newcommand{\zzz}[1]{}}

\usepackage{hyperref}
\hypersetup{       
    unicode=false,          
    pdftoolbar=true,       
    pdfmenubar=true,      
    pdffitwindow=false,     
    pdfstartview={FitH},   
    pdfnewwindow=true,      
    colorlinks=true,      
    linkcolor=blue,          
    citecolor=blue,       
    filecolor=magenta,      
    urlcolor=blue,           
    breaklinks=true
}

\usepackage[toc,page,header]{appendix}
\usepackage{minitoc}

\title{JEPA-Bisim: Learning Robust Visual Representations for Planning with Joint-Embedding Predictive World Models}
\date{}

\author[1]{Leonardo F. Toso$^{*,\dagger}$}
\author[1]{Davit Shadunts$^*$}
\author[1]{Yunyang Lu$^*$}
\author[1]{Gloria Geng$^*$}
\author[2]{Nihal Sharma}
\author[1]{Donglin Zhan}
\author[2]{Nam H. Nguyen}
\author[1]{James Anderson}

\affil[1]{Columbia University, USA}
\affil[2]{Capital One, USA}

\affil[*]{Equal contribution}

\begin{document}

\doparttoc 
\faketableofcontents

\maketitle
\begingroup
\renewcommand{\thefootnote}{\fnsymbol{footnote}}
\footnotetext[2]{Correspondence to: \texttt{leonardo.toso@columbia.edu}.}
\endgroup

\vspace{-1.3cm}
\begin{abstract}
World models learned from high-dimensional visual observations allow agents to make decisions and plan directly in latent space, avoiding pixel-level reconstruction. However, recent latent predictive architectures (JEPAs), including the DINO world model (DINO-WM), display a degradation in test time robustness due to their sensitivity to ``slow features". These include visual variations such as background changes and distractors that are irrelevant to the task being solved. We address this limitation by augmenting the predictive objective with a bisimulation encoder that enforces control-relevant state equivalence, mapping states with similar transition dynamics to nearby latent states while limiting contributions from slow features. We evaluate our model on a navigation task (PointMaze) and on a manipulation task (PushT) under different test-time background changes and visual distractors. Across all benchmarks, our model consistently improves robustness to slow features while operating in a reduced latent space, up to $10\times$ smaller than that of DINO-WM. Moreover, our model is agnostic to the choice of pretrained visual encoder and maintains robustness when paired with DINOv2, SimDINOv2, and iBOT features.
\end{abstract}

\allowdisplaybreaks

\vspace{-0.5cm}
\section{Introduction}
\vspace{-0.2cm}
\label{sec:intro}

World models support planning and control by learning latent dynamics from high-dimensional visual observations. An impactful line of recent work centers on Joint Embedding Predictive Architectures (JEPAs), which eschew reconstruction of raw observations in favor of learning representations by predicting future latent embeddings \citep{grill2020bootstrap, caron2021emerging, lecun2022path, assran2023self, sobal2025learning, balestriero2025lejepa}. Building on JEPAs, the DINO world model (DINO-WM) \citep{zhoudino} learns a latent predictive dynamic by encoding visual observations through a pretrained DINOv2 feature extractor \citep{oquab2023dinov2}. This yields zero-shot planning, with a pretrained, reward-free world model on entirely new tasks.

Despite their empirical success, JEPAs suffer from a fundamental limitation: they pay too much attention to \emph{slowly changing} and {\em task-irrelevant features} \citep{lecun2022path, sobal2022joint}. These are aspects of the raw observations that evolve temporally, such as background textures, lighting, camera viewpoints, and distractors. While they frequently co-occur with the related task-irrelevant visual variations (nuisance factors like dynamic distractors, for example), slow features are distinct in that they have no bearing on the control-relevant transition dynamics. But given their disproportional ease of predictability, they also present a degenerate solution to the JEPA objective by \emph{only} encoding these temporally consistent slow features. Consequently, any downstream world models trained with these poor input representations fail to generalize to test-time environments with \emph{identical} dynamics to those in training, and diverging in the slow features present in the observations. Indeed,  \cite{sobal2022joint} verify this behavior empirically. We also observe this in Figure \ref{fig:maze_motivation}.

Importantly, this limitation \emph{cannot} be fully resolved by larger visual foundation models alone (such as DINOv2 \citep{oquab2023dinov2}, SimDINOv2 \citep{wu2025simplifying}, iBOT \citep{zhou2021ibot}, masked autoencoders \citep{he2022masked}, among others) as they also extract control-irrelevant information. Even with invariant perception modules, the learned latent dynamics may still entangle control-relevant states with nuisance visual features, leading to brittle planning under visual distribution shifts at test time. Ideally, a world model should encode only those aspects of the state that are relevant for predicting future states, while being invariant to task-irrelevant visual factors.\vspace{-0.3cm}
\begin{center}
\fbox{%
\colorbox{gray!12}{%
\parbox{0.98\linewidth}{%
\textbf{Question.} Can we learn latent world models that retain the joint-embedding predictive strengths while also being robust to the slow features?
}}}
\end{center}

\begin{wrapfigure}{r}{0.55\textwidth}
  \centering

  \setlength{\tabcolsep}{2pt}
  \renewcommand{\arraystretch}{0}
\includegraphics[width=0.9\linewidth]{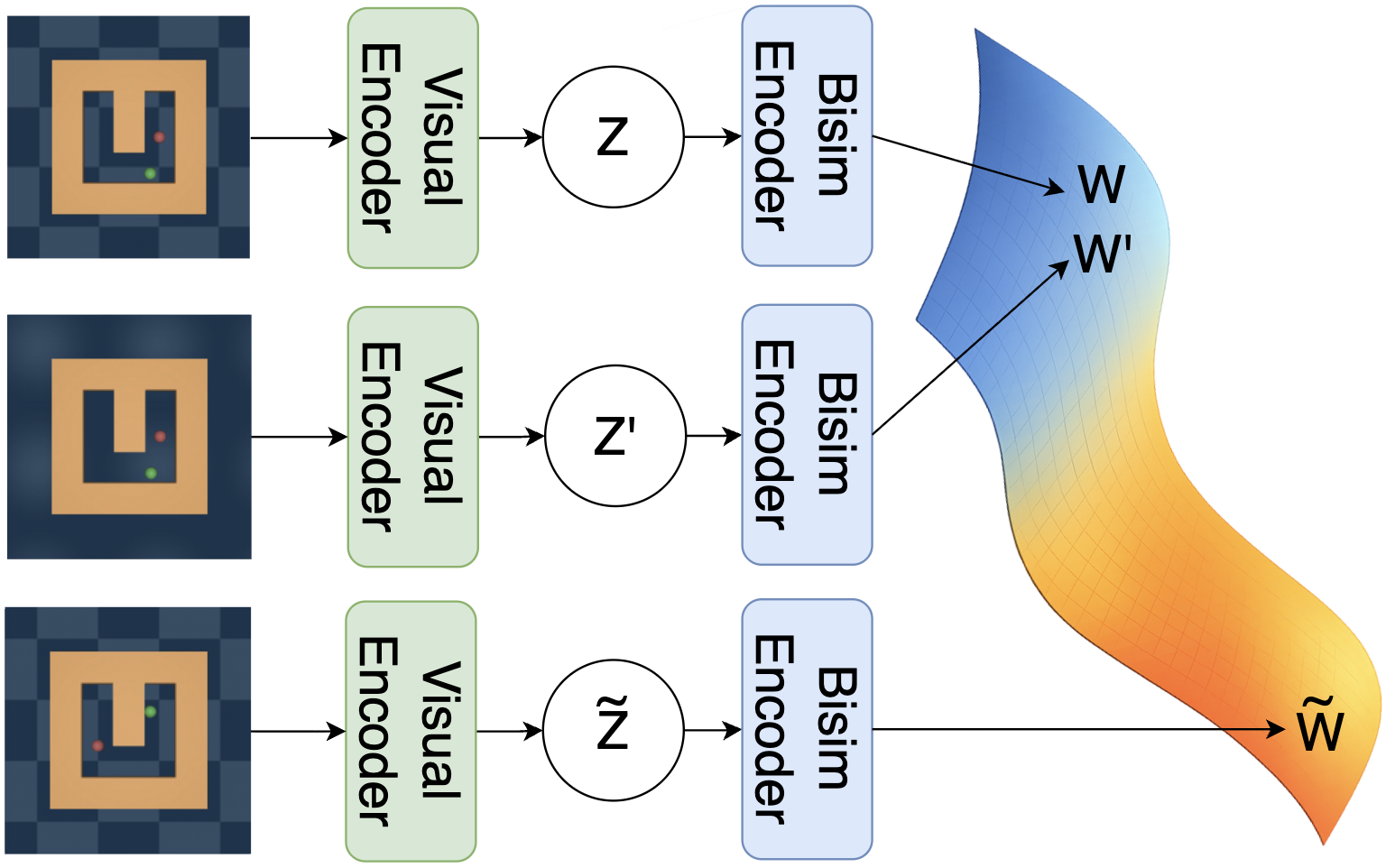} 
\caption{Visually distinct observations that differ only in background (checkerboard and gradient) are first mapped to latent embeddings $Z$ and $Z'$ by a pretrained encoder (initial state in green, goal in red). A bisimulation encoder then projects these into lower-dimensional representations $W$ and $W'$, which are equivalent under on-policy transition dynamics. 
}
\label{fig:encoding_motivation}
\end{wrapfigure}

We address this question using bisimulation metrics for learning invariant representations on top of pretrained visual feature extractors (i.e., DINOv2, SimDINOv2, iBOT). These metrics define state equivalences based on immediate rewards and transition dynamics, and have been studied as a principled foundation for state abstraction in Markov Decision Processes (MDPs) \citep{ferns2004metrics, castro2009equivalence}, transfer learning \citep{castro2010using}, and invariant representation learning \citep{zhanglearning, kemertas2021towards, tkachev2011infinite, shimizubisimulation}. By construction, encoders trained with a bisimulation regularizer discard nuisance features that do not affect control, making it a natural candidate for addressing the slow-feature limitation of JEPAs. \textcolor{black}{Compared to methods that learn distance-aware representations \citep{park2024foundation, wang2026temporal}, our approach specifically targets the adaptation of rich pretrained visual feature spaces (DINOv2, SimDINOv2, iBOT) for planning, leveraging their semantic richness while suppressing control-irrelevant information.}

\begin{center}
\fbox{%
\colorbox{blue!4}{%
\parbox{0.98\linewidth}{%
\textbf{Key Point.} DINO-WM \citep{zhoudino} displays a substantial degradation in performance on relatively simple navigation tasks (e.g., PointMaze \citep{fu2020d4rl}) when evaluated under background changes and visual distractors at test time (see Figure \ref{fig:maze_motivation}). To address this, we propose a \emph{bisimulation encoder} for visual world models that augments latent dynamics with an on-policy bisimulation objective, enforcing invariance in the learned latent space. The bisimulation encoder is trained jointly with the transition model, enforcing states with similar transition behavior to be mapped to nearby latent embeddings while discarding task-irrelevant visual information (see Figure~\ref{fig:encoding_motivation}). Importantly, we also show that the reward signal is not fundamental to learning invariant world-model representations, thus simplifying training.
}}}
\end{center}

\noindent \textbf{Contributions.} Our main contributions are as follows: \\

\noindent $\bullet$  We propose a bisimulation encoder for visual world models that enforces control-relevant state equivalence while discarding task-irrelevant visual features from pretrained visual feature extractors (e.g., DINOv2, SimDINOv2, iBOT).\\

\noindent $\bullet$ We integrate the bisimulation objective into a latent predictive world model architecture (see Figure \ref{fig:archictecture_maze}), showing that invariant representations can be learned jointly with the transition model \emph{without} relying on reward prediction.\\

\noindent $\bullet$ We show that our bisimulation encoder yields robust planning performance under several background changes and moving distractors for navigation and manipulation tasks, with latent embeddings dimensions that are $10\times$ smaller than DINO-WM. \\

\begin{wrapfigure}{r}{0.6\textwidth}
\vspace{-0.8cm}
  \centering
   \includegraphics[width=1\linewidth]{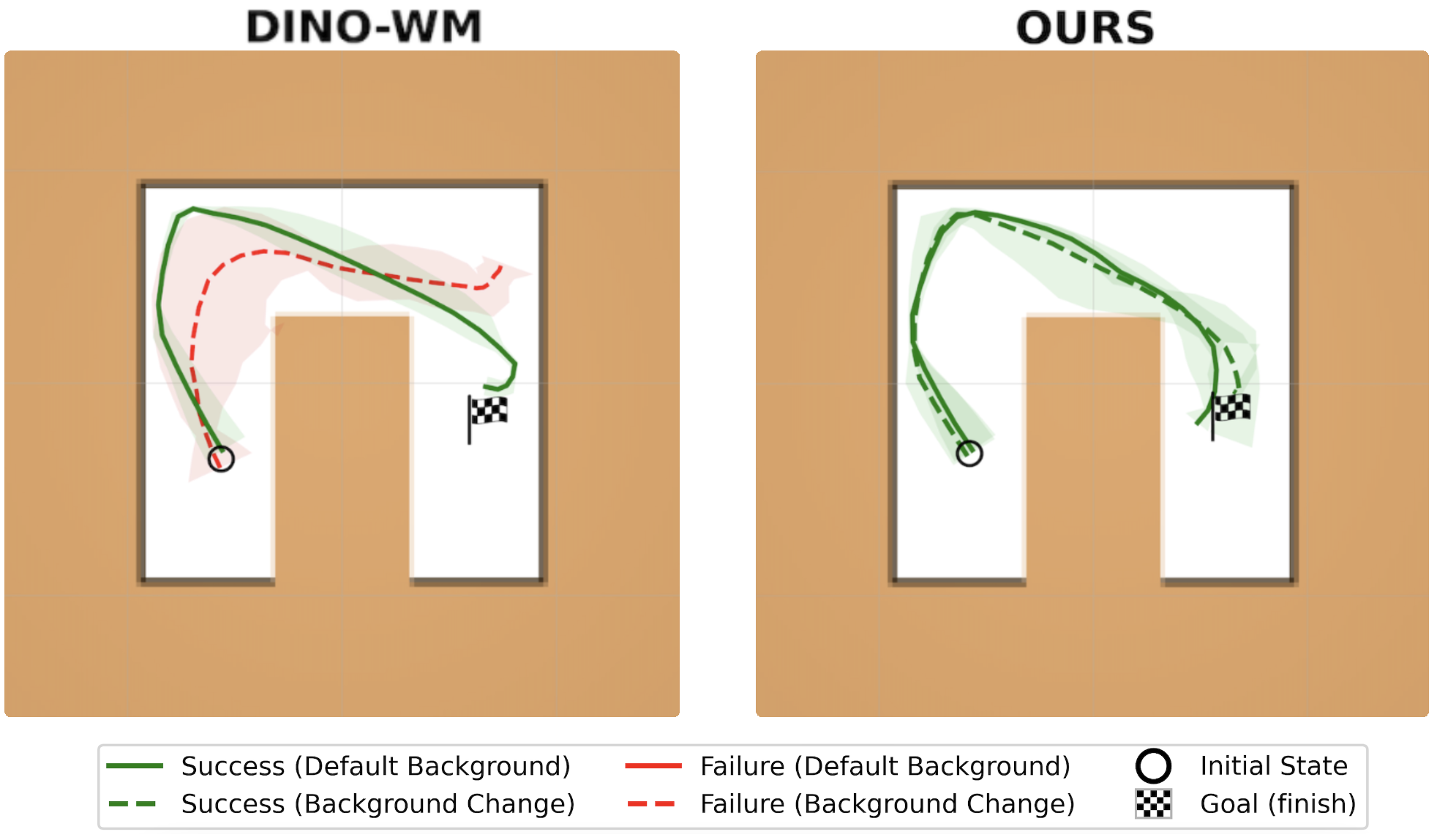}
\caption{Rollouts for PointMaze navigation under a background change at test time. For clarity, the background change is not depicted in the figure. See Section~\ref{sec:backgrounds} for details on background changes. While DINO-WM fails to reach the goal due to background change our model succeeds.}
\label{fig:maze_motivation}
\end{wrapfigure}

Apart from the works referenced above, we also draw inspiration from the myriad works on joint embedding  predictive architectures, introduced in \cite{lecun2022path,sobal2025learning}, as well as the seminal works of \cite{castro2009equivalence,castro2010using,zhanglearning} on bisimulation metrics and invariant representations. As a result, this paper is also closely related to work at the intersection of model-based control and latent world models \citep{sun2024learning,lim2025jedi}.  We refer the reader to Appendix \ref{apx:related work} for a more detailed discussion.

\section{Preliminaries}

Throughout, we consider a vision-based control task given by a partially observable MDP (POMDP) that evolves in discrete time $t \in \left\{0,1, 2,\dots\right\}$. An observation is an RGB image $o_t \in \mathcal{O}$, an action is denoted by $a_t \in \mathcal{A}$, and a trajectory is a sequence $\left(o_t,a_t,o_{t+1} , a_{t+1}, \dots \right)$. We denote the data distribution induced by a behavior policy $\pi$ by $\mathcal{D}_\pi$, meaning that $\left(o_t,a_t,o_{t+1}\right) \sim \mathcal{D}_\pi$ when rollouts are generated by playing the environment with $\pi$. We use $\mathbb{E}_{\mathcal{D}_\pi}\left[\cdot\right]$ to denote the expectation with respect to this data distribution. \textcolor{black}{We use $\mathcal{B}$ to denote the replay buffer (distinct from the data distribution $\mathcal{D}_\pi$); the replay buffer stores previously collected transitions and is used for pairwise comparisons in the bisimulation loss (see Appendix~\ref{appendix:archictecture}).}

\begin{figure*}
  \centering
  \includegraphics[width=1\linewidth]{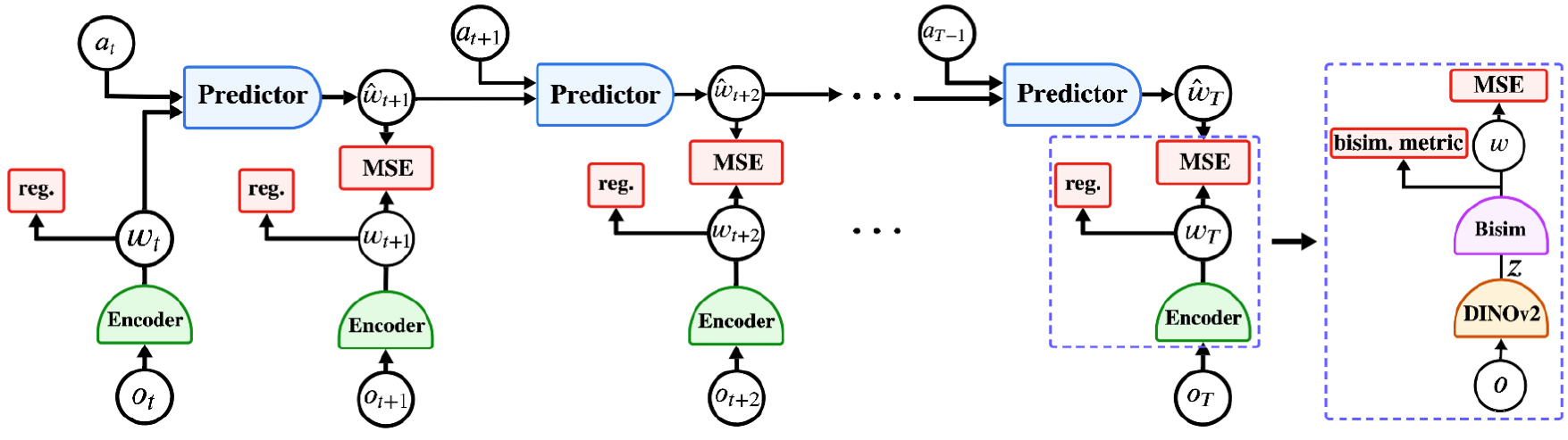}
\caption{Model architecture and training objectives. }
\label{fig:archictecture_maze}
\end{figure*}

\subsection{Joint-Embedding Predictive Architectures}\label{subsec:jepa}

JEPAs learn a visual encoder $f_\theta:\mathcal{O}\to\mathbb{R}^{d_z}$ that maps an observation $o_t$ to a latent embedding, i.e., $ z_t = f_\theta\left(o_t\right)\in\mathbb{R}^{d_z},$ where $d_z$ is the embedding dimension and $\theta$ is the vector of encoder parameters. Jointly, a predictor (i.e., latent dynamics model) $\tilde{T}_\phi:\mathbb{R}^{d_z}\times\mathcal{A}\to\mathbb{R}^{d_z}$ is trained to map the current latent state and action to a predicted next latent state $ \hat z_{t+1} = \tilde{T}_\phi\left(z_t,a_t\right),$ with parameters $\phi$.
 
Since representations are defined up to arbitrary invertible transformations, JEPA training includes a target embedding.
The target is denoted by  $z_{t+1} = f_{\theta}\left(o_{t+1}\right)$.
The JEPA objective minimizes the prediction loss between $\hat z_{t+1}$ and $z_{t+1}$,
\begin{align}\label{eq:jepa_pred_loss}
\hspace{-0.5cm}\mathcal{L}_{\mathrm{pred}}\left(\theta,\phi\right) =
\mathbb{E}_{\mathcal{D}_\pi}\left[
\ell\left(\tilde{T}_\phi\left(f_\theta\left(o_t\right),a_t\right), f_{\theta}\left(o_{t+1}\right)\right)
\right],
\end{align}
where $\ell:\mathbb{R}^{d_z}\times\mathbb{R}^{d_z}\to\mathbb{R}_+$ is the mean squared error.
 
We emphasize that DINO-WM \citep{zhoudino}, follow this principle, with the key distinction that the visual encoder parameters $\theta$ are fixed and not learned during training. \textcolor{black}{The motivation for using frozen pretrained encoders rather than learning from scratch is twofold: (i) large-scale pretrained visual features (DINOv2, SimDINOv2, iBOT) provide semantically rich representations that improve sample efficiency for dynamics learning, and (ii) fixing the encoder decouples representation learning from dynamics learning, enabling modular composition with different visual backbones. The bisimulation encoder then adapts these fixed features to be control-relevant, combining the best of both worlds.} Therefore, observations are first mapped to latent embeddings using a pretrained visual foundation model (e.g., DINOv2 \citep{oquab2023dinov2}), and a predictive dynamics model is trained on top of that latent space. Here, we also treat the pretrained visual encoder as fixed and train an extra bisimulation encoder jointly with the dynamics on top of it.
\vspace{0.1cm}
 
\begin{remark}[Slow Features]
We emphasize that JEPAs inadvertently focus on the slow features. In particular, when observations contain static backgrounds or persistent distractors that vary slowly relative to task-relevant dynamics, the prediction loss \eqref{eq:jepa_pred_loss} can be minimized by encoding such features into the latent embedding $z_t = f_\theta(o_t)$, since they induce small temporal differences $\|z_{t+1}-z_t\|$. As a result, slowly varying nuisance information may dominate the predictive objective, leading to latent embeddings that fail to capture the underlying task structure \citep{sobal2022joint} and thus are not generalizable to background changes at test time (see Figure \ref{fig:maze_motivation}). This limitation motivates the use of bisimulation. By defining latent distances in terms of on-policy transition behavior, bisimulation induces representations that collapse nuisance visual features while preserving the task-relevant information for planning.
\end{remark}

\subsection{Bisimulation Metric}\label{subsec:bisimulation}

Let us now (informally) revisit the definition of on-policy bisimulation metrics from \cite{castro2020scalable}. Intuitively, two latent states, in the pretrained visual encoder space, $z$ and $z^\prime$ are ``bisimilar" if they induce similar immediate rewards (for optimality equivalence) and similar future state evolution (for dynamics equivalence) under the same actions.
 
Here, the environment is partially observed and represented through latent embeddings obtained from a visual encoder (e.g., DINOv2, SimDINOv2, iBOT). The goal is to learn a lower-dimensional representation $w = h_\eta(z) \in \mathbb{R}^{d_w}$, with $d_w \ll d_z$, that preserves control-relevant structure while discarding nuisance visual variation. We formalize state equivalence with the \emph{on-policy} bisimulation metric \citep{castro2020scalable}. 
 
\begin{definition}[On-Policy Bisimulation Metric \citep{castro2020scalable}]\label{def:onpolicy_bisim}
Given a policy $\pi$, a reward function $r(z,a)$, and a transition kernel $P_\pi(\cdot\mid z)$ induced by $\pi$, the on-policy bisimulation metric is 
\begin{align*}
d_\pi(z,z') = \underbrace{\left| r_\pi(z) - r_\pi(z') \right|}_{\mathrm{reward~similarity}}
+ \gamma  \underbrace{W_1 \left( P_\pi(\cdot \mid z), 
P_\pi(\cdot \mid z') \right)}_{\mathrm{dynamics~ similarity}}, \text{ with } r_\pi(z) = \mathbb{E}_{a \sim \pi(\cdot \mid z)}[r(z,a)],
\end{align*}
where $\gamma \in (0,1)$ is the discount factor and ${W}_1$ denotes the $1$-Wasserstein distance.
\end{definition}

\begin{center}
\fbox{%
\colorbox{blue!4}{%
\parbox{0.98\linewidth}{%
\textbf{On Dropping the Reward Term.}  Note that $d_\pi(z,z')$ has two terms: a reward similarity term $\lvert r_\pi(z) - r_\pi(z') \rvert$ and a dynamics similarity term measured via the $1$-Wasserstein distance \citep{givens1984class}. The reward term ensures that if two state trajectories are similar, they are not collapsed if they yield different rewards. In particular, such reward discrepancy is crucial when bisimulation is used for policy transfer or value-function approximation \citep{ferns2004metrics, castro2009equivalence, castro2010using}. However, for the purpose of learning an invariant representation for JEPAs, we claim that this term is not fundamental. That is because the transition term alone induces a meaningful notion of equivalence: states that exhibit similar evolution under the same action sequences are mapped to nearby latent representations in the bisimulation latent space (see our illustration in Figure \ref{fig:encoding_motivation}).
}}}
\end{center}

Moreover, dropping the reward term is well aligned with JEPAs, which \textcolor{black}{are trained without reward signal on offline trajectory data and rely on predictive consistency in latent space rather than task-specific reward structure. Incorporating rewards would require task-specific supervision and would break the modularity between representation learning and zero-shot downstream planning, i.e., a core design principle of JEPAs \citep{lecun2022path}.} Accordingly, we omit the reward term in our implementation and rely only on approximating the Wasserstein transition similarity to enforce invariance. This design choice is also supported numerically in Appendix \ref{appendix:reward_aware}, where a reward predictor is included.

\subsection{Learning Invariant Representations}

Now, we discuss the invariant representation learning problem for our joint-embedding predictive world model. Recall that $z_t=f_\theta \left(o_t\right)\in\mathbb{R}^{d_z}$ typically contains both control-relevant information and nuisance features. Our goal is to learn an  encoder $w_t = h_\eta\left(z_t\right)\in\mathbb{R}^{d_w},
$ that limits the contribution of the slow features in the transition dynamics.
 
To do so, we jointly learn the dynamics in the bisimulation space through
$T_\phi:\mathbb{R}^{d_w}\times\mathcal{A}\to\mathbb{R}^{d_w},$ such that $ \hat w_{t+1}=T_\phi \left(w_t,a_t\right)$ predicts the next latent state. With the notation set up in Section \ref{subsec:jepa}, we have the predictive dynamics augmented with the bisimulation encoder, i.e., $T_\phi(h_\eta(f_\theta(o_t)),a_t)$, where again $\theta$ is fixed.
 
To connect invariance to bisimulation, we consider an abstract state space $\mathcal{W}\subseteq\mathbb{R}^{d_w}$ equipped with the metric $d_\pi$. We aim to learn $h_\eta$ and $T_\phi$ jointly so that the distances in $\mathcal{W}$ approximate an on-policy bisimulation metric, meaning that states that are equivalent for control under rollouts are mapped nearby in such latent space. We do this by optimizing a bisimulation objective that aligns pairwise latent distances with a target quantity that compares on-policy dynamics. Concretely, for a pair of observation, action, and next observation $\left(o_t,a_t,o_{t+1}\right)$ and $\left(o'_t,a'_t,o'_{t+1}\right)$ drawn from $\mathcal{D}_\pi$, we have the latent states
$$
\underbrace{z_t=f_\theta \left(o_t\right),\; z'_t=f_\theta \left(o'_t\right)}_{\text{Pretrained latents}},\;
\underbrace{w_t=h_\eta\left(z_t\right),\; w'_t=h_\eta\left(z'_t\right)}_{\text{Bisimulation latents}}. 
$$

The transition dynamics $T_\phi$ is trained using
\begin{align*}
\hspace{-0.2cm}\mathcal{L}_{\mathrm{dyn}}\left(\eta,\phi\right)
\triangleq \mathbb{E}_{\mathcal{D}_\pi}\hspace{-0.1cm}\left[
\left\|h_\eta\left(f_\theta\left(o_{t+1}\right)\right)\hspace{-0.05cm}-\hspace{-0.05cm}T_\phi\left(h_\eta\left(f_\theta \left(o_t\right)\right),a_t\right)\right\|^2
\right]\hspace{-0.05cm},
\end{align*}
and we note that, in practice, as we do not have $\mathcal{D}_\pi$, we compute the empirical mean (MSE) instead.
 
On the other hand, for invariance, our loss encourages the distance $\left\|w_t-w'_t\right\|_2$ to match a control-relevant similarity signal measured by the bisimulation metric. We define the on-policy one-step bisimulation target between states as $\Delta_{\mathrm{bisim}}(\eta,\phi) \triangleq \gamma \left\|T_\phi\left(w_t,a_t\right)- T_\phi\left(w'_t,a'_t\right)\right\|_2,$
where it compares the predicted next latent states under the learned dynamics. \textcolor{black}{Note that the actions $a_t$ and $a'_t$ are drawn independently from the two separate trajectories sampled from the replay buffer $\mathcal{B}$; they are not required to be the same. This is consistent with the on-policy bisimulation formulation of \cite{castro2020scalable}, where the metric averages over the policy's action distribution. In practice, using independently sampled actions provides a diverse set of comparisons across the replay buffer, which stabilizes training.} We fit the bisimulation encoder by minimizing the following loss:
$$
\mathcal{L}_{\mathrm{bisim}}\left(\eta,\phi \right) \triangleq \mathbb{E}_{\mathcal{D}_\pi}\left[ \left(\left\|w_t-w'_t\right\|_2-\Delta_{\mathrm{bisim}}(\eta,\phi)\right)^2 \right].
$$
Finally, we define the overall training objective as follows:
\begin{align}
\hspace{-0.35cm}\mathcal{L}_{\mathrm{jepa-bisim}}\left(\eta,\phi\right) \triangleq 
\mathcal{L}_{\mathrm{dyn}}\left(\eta,\phi\right)
+\lambda_{\mathrm{bisim}}\mathcal{L}_{\mathrm{bisim}}\left(\eta,\phi\right),
\label{eq:overall_prelim_obj}
\end{align}
with $\lambda_{\mathrm{bisim}}>0$ being the regularization weight.

\begin{wrapfigure}{r}{0.6\textwidth}
\centering
\vspace{-0.2cm}
\includegraphics[width=0.6\textwidth]{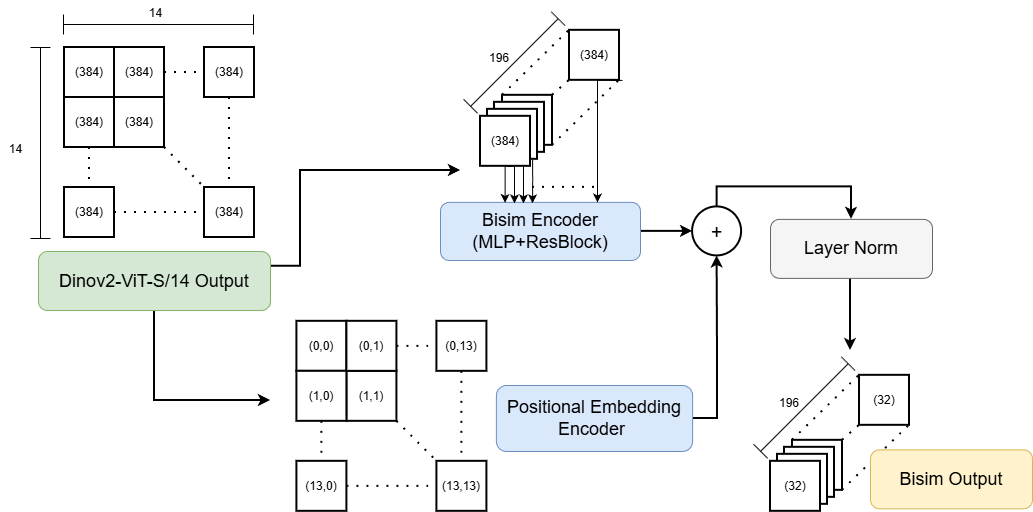}
\caption{Bisimulation encoder with patch-based design, with $196$ of patches and $32$ output patch dimension.}
\label{fig:bisim_architecture}
\vspace{-1.3cm}
\end{wrapfigure}
 
Minimizing the loss formalizes our goal of learning a compact latent space that preserves on-policy predictive structure for planning while suppressing task-irrelevant visual features. We next describe the architecture, loss instantiation, and planning in the learned bisimulation space.

\section{Our World Model}\label{sec:our_world_model}

We now describe our architecture for learning invariant world models that builds on pretrained visual features (DINOv2, SimDINOv2, and iBOT) and augments them with a bisimulation encoder to suppress task-irrelevant information. The model learns latent transition dynamics and performs planning with CEM. An overview of the pipeline objectives is shown in Figure \ref{fig:archictecture_maze}.

\subsection{Pretrained Encoders}

In our experiments, we consider three self-supervised visual encoders: DINOv2 \citep{oquab2023dinov2} (with ViT-S/14), SimDINOv2 \citep{wu2025simplifying} (with ViT-B/16), and iBOT \citep{zhou2021ibot} (with ViT-S/16). These are trained on large-scale image datasets using joint embedding or masked prediction objectives and have been shown to produce semantically rich visual features \citep{jiang2023clip}. \textcolor{black}{We include SimDINOv2 as a variant of DINOv2 that incorporates coding-rate regularization for improved feature quality, and iBOT as a masked-prediction alternative. These pre-trained encoders allow us to assess whether the bisimulation encoder's benefits are agnostic to the choice of visual backbone. When using DINOv2 features, we refer to our model as DINO-Bisim.
}

\subsection{Bisimulation Encoder}

Given pretrained visual embeddings $z_t$, we learn a bisimulation encoder
$h_\eta : \mathbb{R}^{N_p \times d_z} \rightarrow \mathbb{R}^{N_p \times d_w},$ where $ w_t = h_\eta(z_t),$ as discussed in Section \ref{subsec:bisimulation}. We note that in contrast to \cite{zhanglearning, shimizubisimulation}, our bisimulation encoder is designed in patches (we refer the reader to Appendix \ref{appendix:patch_design} for details). Such encoder consists of a per-patch residual MLP projection head (Linear-ResMLP-Linear) that maps each DINOv2 patch embedding (or raw image patch) into a $d_w = 32$ dimensional latent, followed by learned spatial positional embeddings and LayerNorm.

This encoder is trained to ensure that distances in the $w$-space reflect on-policy bisimulation distances, as defined in \eqref{eq:overall_prelim_obj}. To this end, we \emph{flatten} the spatial patch features prior to computing the previously defined losses and compare states across different trajectories that are stored in a replay buffer $\mathcal{B}$, following \cite{zhanglearning}. However, minimizing the bisimulation loss admits degenerate solutions as mapping all states to a single representation trivially minimizes pairwise distances, and we refer to this as the representation \emph{collapse}. To circumvent such collapse, prior work \citep{bardes2021vicreg, zhu2023variance} has adopted variance-invariance-covariance regularization (VICReg) to guarantee that each latent dimension keeps sufficient variance over batches of states (i.e., to guarantee that the latents are sufficiently spread). We refer the reader to \citep[Section 4.1]{bardes2021vicreg} for an excellent overview of the approach. Here, invariance is given by the bisimulation loss.

\begin{wrapfigure}{r}{0.6\textwidth}
\centering
\vspace{-0.3cm}
\includegraphics[width=0.6\textwidth]{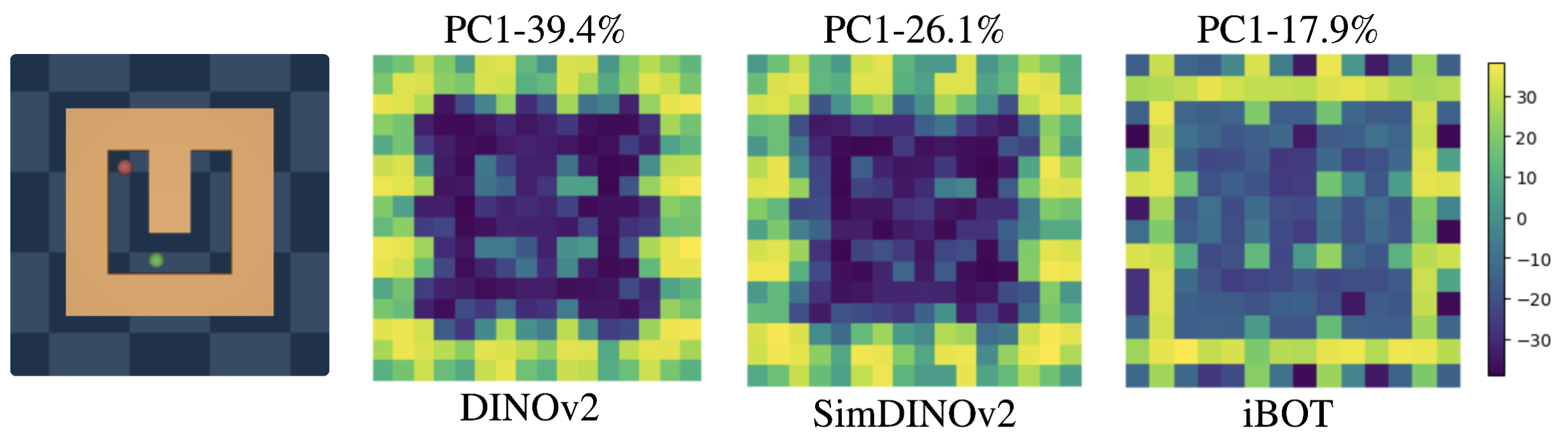}
\caption{\textcolor{black}{Visualization of the first principal component (PC1) of latent embeddings produced by different visual feature encoders for a PointMaze observation. In all cases, PC1 predominantly encodes background and layout information rather than control-relevant features, motivating our PCA-based VC-Reg in the bisimulation encoder.}}
\label{fig:pca-pointmaze}
\end{wrapfigure}
 
In our setting, however, the standard VICReg \citep{bardes2021vicreg} \emph{may be} insufficient. This is because, for the pretrained visual encoders considered in this work (i.e., DINOv2, SimDINOv2, and iBOT) a large fraction of the variance is often concentrated in a small number of principal components that correlate with the slow features rather than task-relevant dynamics. Therefore, imposing a uniform variance across all coordinates may preserve nuisance visual features, preventing the bisimulation encoder from collapsing the task-irrelevant information.
 
Leveraging this insight, we propose a PCA-based variance-covariance regularization (VC-Reg) loss, which adapts VICReg to pretrained visual embeddings by guaranteeing a minimum variance in directions that are not dominated by a large fraction of slow features. \textcolor{black}{Throughout the paper, we use VICReg to refer to the original formulation of \cite{bardes2021vicreg}, which enforces variance, invariance, and covariance regularization uniformly across all latent dimensions. We use VC-Reg (or PCA-VC-Reg) to refer to our variant, which applies variance and covariance regularization selectively in the PCA basis, omitting the invariance term (handled by the bisimulation loss instead).}

\subsection{PCA-based Variance Covariance Regularization}\label{subs:PCA-loss}

As illustrated in Figure \ref{fig:pca-pointmaze}, the pretrained visual embeddings from DINOv2, SimDINOv2, and iBOT allocate a large amount of variance to the first principal component (PC1), which are correlated with task-irrelevant features coming from backgrounds and distractors. \textcolor{black}{This concentration of task-irrelevant variance in the leading PC is a common property of self-supervised visual encoders: as noted in the DINOv2 paper \citep{oquab2023dinov2}, background information tends to dominate the first principal component and can be attenuated by thresholding it. While this property may not hold universally for all encoders or environments, it is observed consistently for DINOv2, SimDINOv2, and iBOT in our PointMaze setting (see Figure~\ref{fig:pca-pointmaze} and Appendix~\ref{appendix:PCA-loss}).} If we use the standard VICReg loss with uniform variance across the coordinates of the bisimulation latent space, the optimizer may still preserve such irrelevant directions simply to satisfy the variance threshold, conflicting with the goal of learning a representation invariant to those slow features. Therefore, we do VC-Reg in the PCA basis of the current batch (see Appendix \ref{appendix:PCA-loss} for more details). 

We emphasize that if the visual encoder does not isolate control-irrelevant features within a dominant principal component, the PCA-based VC-Reg regularization becomes \emph{optional}. In this case, the bisimulation-based encoder alone can mitigate, though only partially, the influence of task-irrelevant variations in the learned representation.

\subsection{Transition Model}

We model latent dynamics using a vision transformer that predicts the next state in the bisimulation latent space. Following DINO-WM \citep{zhoudino}, the vision transformer architecture \citep{dosovitskiy2020image} is suitable for processing the underlying patch-based embeddings. Since the bisimulation latent states are already represented as sequences of patch embeddings, we remove the ViT tokenization layer, \textcolor{black}{yielding a transformer architecture operating directly on patch-level bisimulation features without a separate tokenization step (analogous to a ``decoder-only'' design in LLMs, where the transformer processes tokens causally without a separate encoder).}
 
The transition model takes as input a history of past bisimulation latent states and actions across multiple trajectories, $\{\left(w_{j, t-H:t-1}, a_{j, t-H:t-1}\right)\}_{j=1}^{N_\text{traj}},$ where $H$ denotes the context length and $N_{\text{traj}}$ the number of trajectories, and it predicts the next latent state $\hat w_t$ (i.e., in the bisimulation space). Note that the subscript $j$ denotes the trajectory embedding and action at time step $t$. To respect temporal causality, a causal attention mask is implemented (as in \cite{zhoudino}) ensuring that predictions at time $t$ depend only on past latent states and actions. More precisely, each patch embedding $w_t^i$ attends to the corresponding patch embeddings $\{w_{t-H:t-1}^i\}_{i=1}^{N_p}$ over the context. Predicting all patches jointly yields global spatial structure and temporal dynamics that improve temporal generalization as discussed in \cite{zhoudino}.
 
By ``filtering out" task-irrelevant visual features, the bisimulation encoder enables the transition model to learn dynamics that are \textcolor{black}{invariant to nuisance visual factors (backgrounds, distractors), focusing only on control-relevant state transitions.} The learned joint-embedding predictive world model is then used for planning via MPC \citep{kouvaritakis2016model}, with actions selected using the cross-entropy method (CEM).

\subsection{Planning}

Given an initial observation $o_0$ and a goal observation
$o_{\mathrm g}$, we encode
$w_0=h_\eta(f_\theta(o_0))$ and the goal
$w_{\mathrm g}=h_\eta(f_\theta(o_{\mathrm g}))$.
For a fixed goal, we define the latent goal-reaching cost as follows:
\begin{equation}
\label{goal_reaching_cost}
c(w)\triangleq \left\|w-w_{\mathrm g}\right\|_2^2.
\end{equation}
We optimize the action sequence
$\boldsymbol a=(a_0,\ldots,a_{T-1})$ through the learned
latent dynamics:
\begin{equation}
\begin{aligned}
\min_{\boldsymbol a}\quad
& c(\hat w_T)\quad  \text{ subject to}\quad
 \hat w_{t+1}=T_\phi(\hat w_t,a_t),
\quad t=0,\ldots,T-1, \quad  \hat w_0=w_0.
\end{aligned}
\end{equation}
We solve this terminal-cost objective using CEM with a planning horizon of $T=5$ in a receding-horizon MPC scheme, executing the first action and replanning at each step. Our bisimulation latent space is $10\times$ smaller than DINO-WM's, reducing the computational cost of latent rollouts.

\subsection{A Generalization Bound} \label{sec:generalization}

Before turning to the experiments, we also provide a theoretical justification for learning bisimulation-based representations on top of JEPA latents. Let us first define the $T$-step on-policy cost-to-go as follows:\vspace{-0.1cm}
\begin{align*}
\hspace{-0.2cm} J_T^\pi(z) \triangleq \mathbb{E}\left[\sum_{t=0}^{T-1}{\gamma^t c\left(h_\eta(z_t)\right) \mid z_0=z}\right], \; z_{t+1}\sim P_\pi(\cdot \mid z_t),
\end{align*}
where the expectation is taken over trajectories $\{z_t\}_{t=0}^{T}$ induced by policy $\pi$ and transition kernel $P_\pi$ starting from a visual pretrained latent state $z_0=z$. We defer the proof of Theorem \ref{thm:rf_generalization} to Appendix \ref{appendix:proof_generalization}.

\begin{theorem}[Reward-free planning generalization bound]\label{thm:rf_generalization} Suppose $h_\eta(\cdot)$ $\varepsilon$-approximately preserves the reward-free bisimulation metric in the sense that
$$
\left|\left\|h_\eta(z)-h_\eta(z')\right\|_2-d_\pi(z,z')\right|\le \varepsilon, \; \text{ for any } z,z'.
$$
 
In addition, suppose that $h_\eta(\cdot)$ is uniformly bounded as  $\left\|h_\eta(z)\right\|_2\leq {H}_w$ for all $z$ and thus $\left\|w_g\right\|_2\leq {H}_w$. Then for any two pretrained visual latent embeddings $z,z'$,
\begin{align}\label{eq:rf_value_bound}
\left|J_T^\pi(z)-J_T^\pi(z')\right| \leq 8 {H}_wTd_\pi(z,z'). 
\end{align}
\end{theorem}

We also note that the linear dependence on the horizon $T$ contrasts with the $(1-\gamma^{T})/(1-\gamma)$ dependence from \citep{zhanglearning, shimizubisimulation}. This difference arises from dropping the reward term in the bisimulation metric: while reward-aware bisimulation enables geometric discounting of cost discrepancies through immediate reward alignment, a reward-free metric propagates the discrepancies through the transition dynamics, yielding a cumulative horizon-dependent bound. In this sense, the $T$ factor can be interpreted as the price paid for enforcing control-relevant invariance without task-specific reward supervision.

\begin{center}
\fbox{%
\colorbox{blue!4}{%
\parbox{0.98\linewidth}{%
\textbf{Key Takeaway.} The above result implies that if two latent states in the pretrained visual encoder space differ only due to nuisance visual variations, such as slow features, then their induced planning behavior remains similar. More concretely, let $z = f_\theta(o)$ embed a PointMaze scene with a checkerboard background, and let $z' = f_\theta(o')$ embed the same scene with a gradient background (see Figure \ref{fig:encoding_motivation}), since these visual changes do not affect on-policy transition behavior, the bisimulation distance $d_\pi(z, z')$ is small. Consequently, the difference in the induced planning cost is bounded by $\mathcal{O}\left(d_\pi(z, z')\, T\right)$. Therefore, robustness of planning when starting from $w$ or $w'$ in the bisimulation latent space is governed by their distance in that space.
}}}
\end{center}

\section{Experiments}\label{sec:experiments}

We now evaluate\footnote{Code to reproduce the results can be found at \url{https://github.com/jd-anderson/dino_bsmpc}.} our world model by comparing it with DINO-WM \cite{zhoudino} and domain randomization (DR) \cite{tobin2017domain} for a navigation task (PointMaze) and a manipulation task (PushT) under different visual conditions, and pretrained visual encoders. Further details on the task setup are provided in Appendix \ref{appendix:tasks}.\\

\noindent \textbf{Evaluation Metric.} We evaluate each condition using $50$ independently sampled
initial and goal pairs and define success rate (SR) as the fraction of successful episodes. We report the absolute SR under the native background (NC) and the relative degradation under each shifted condition $s$, defined as $\Delta_s = 100\times (1-\mathrm{SR}_s/\mathrm{SR}_{\mathrm{NC}})$.
Positive values indicate a percentage decrease relative to the same model's native-background performance, while negative values indicate improvement. We also report the mean relative degradation across the five shifted conditions
(SC, C, LC, LCG, and D). Lower degradation is better.

\subsection{Baseline}

\begin{wrapfigure}{r}{0.6\textwidth}
\centering
\vspace{-0.4cm}

\includegraphics[width=0.6\textwidth]{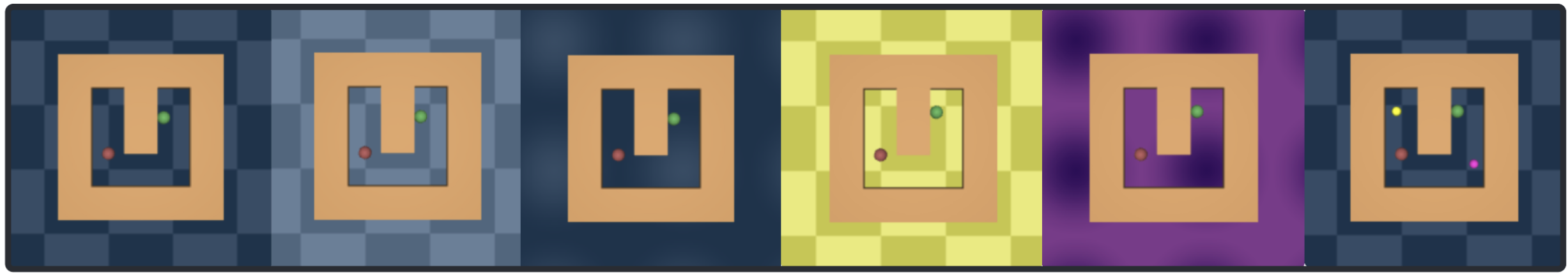}

\vspace{0.08cm}

\includegraphics[width=0.6\textwidth]{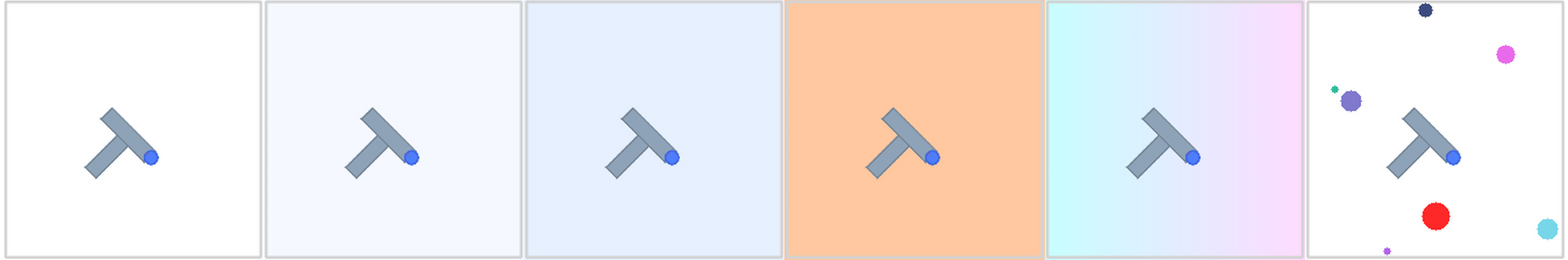}

\vspace{-0.3cm}
\caption{
Visual distribution shifts used to evaluate robustness on
\textbf{PointMaze} (top) and \textbf{PushT} (bottom).
From left to right, the conditions are:
\textbf{NC}: No Change,
\textbf{SC}: Slight Change,
\textbf{C}: Color Gradient,
\textbf{LC}: Large Color Change,
\textbf{LCG}: Large Color Gradient,
and \textbf{D}: moving Distractors.
}
\label{fig:visual_shifts}

\vspace{-0.5cm}
\end{wrapfigure}

We compare our model with two baselines: DINO-WM \cite{zhoudino} and domain DR \cite{tobin2017domain}. Both baselines are evaluated under the same task, planning horizon, and evaluation metric previously described.\\

\noindent \textbf{DINO-WM.} We use the original DINO-WM architecture as proposed in \cite{zhoudino}. The model operates directly on pretrained DINOv2 patch embeddings with $196$ patches, each of dimensionality $384$, and learns latent dynamics without explicit invariance regularization. DINO-WM is trained with a batch size of $32$.

\subsection{Background Changes}\label{sec:backgrounds}

\noindent \textbf{Domain Randomization (DR).} We also train DINO-WM with domain randomization \cite{tobin2017domain}. In particular, we randomize the visual appearance of the task for $40\%$ of the training dataset by applying background color changes and gradient-based background modifications, for the PointMaze task. To avoid introducing spurious temporal artifacts, the randomized background is kept consistent across trajectories. We refer the reader to Appendix \ref{appendix:domain_randn}.

Our model adopts the same number of visual patches (i.e., $196$) as DINO-WM, but introduce a bisimulation encoder that maps each patch to a $32$-dimensional latent embedding (see Figure \ref{fig:bisim_architecture}). This leads to a latent space that is $10\times$ lower-dimensional than the $384$-dimensional patch embeddings used by DINO-WM. Our model is trained with a batch size of $20$. We refer the reader to Appendix \ref{appendix:hyperparamters} for details on the model hyperparameters.

\subsection{Results}

\begin{table}
\centering
\vspace{-0.9cm}
\caption{\textbf{PointMaze.} Absolute success rate under
NC ($\uparrow$) and relative degradation (\%, $\downarrow$)
under shifts. Negative values indicate improvement.
Avg.: mean degradation over the five background shifts.}
\label{tab_main_results22}
\label{tab:2ndmain_results}
\small
\setlength{\tabcolsep}{3pt}
\resizebox{\linewidth}{!}{%
\begin{tabular}{lccccccc}
\toprule
\shortstack[c]{\textbf{Model} } & \shortstack[c]{\textbf{NC} \\ (Abs.) }
& \shortstack[c]{\textbf{SC} \\ (Rel.\%) } & \shortstack[c]{\textbf{C} \\ (Rel.\%) } & \shortstack[c]{\textbf{LC} \\ (Rel. \%) }
& \shortstack[c]{\textbf{LCG} \\ (Rel.\%) } & \shortstack[c]{\textbf{D} \\ (Rel.\%) } & \shortstack[c]{\textbf{Avg} \\ (Rel.\%) }\\
\midrule
\multicolumn{8}{l}{\textit{Comparison with DR}} \\
\midrule
DINO-WM
& 0.8 & 10.0 & 25.0 & 30.0 & 40.0 & 2.5 & 21.5 \\
DINO-WM w/DR
& 0.82 & 0.0 & 0.0 & 17.1 & 22.0 & 0.0 & 7.8 \\
\textbf{Ours (DINOv2)}
& 0.78 & $-2.6$ & 2.6 & $-10.3$ & 0.0 & $-5.1$
& $\textcolor{blue}{\mathbf{-3.1}}$ \\
\midrule
\multicolumn{8}{l}{\textit{Encoder ablation}} \\
\midrule
End-to-End
& 0.68 & 35.3 & $-2.9$ & 61.8 & 47.1 & 5.9 & 29.4 \\
SimDINOv2
& 0.40 & 5.0 & 10.0 & $-5.0$ & $-5.0$ & 10.0 & 3.0 \\
iBOT
& 0.72 & 2.8 & $-2.8$ & 0.0 & 0.0 & 0.0 & $\textcolor{blue}{\mathbf{0.0}}$ \\
\bottomrule
\end{tabular}%
}
\end{table}

We consider different background changes and  moving distractors (see Figure \ref{fig:visual_shifts}) that alter the visual appearance of the scene while preserving the underlying transition dynamics. These changes are applied only at test time, unless otherwise stated, and are designed to isolate the effect of slow features on the planning performance. \textcolor{black}{The moving distractors (colored dots) move continuously within the maze but do not collide with or affect the agent.
}

Table~\ref{tab_main_results22} reports the absolute success rate under NC and relative degradation under background changes and a moving distractor for PointMaze. Negative
degradation indicates improvement over the same model's native-background performance.\\

\noindent \textbf{Robustness to the background.} DINO-WM degrades when the visual background differs from the one seen during training, with relative degradation ranging from $10.0\%$ under SC to $40.0\%$ under LCG. In contrast, our model largely preserves its
native-background performance, with degradation ranging from $-10.3\%$ to $2.6\%$ across background changes. These results are consistent with the bisimulation encoder
reducing sensitivity to background-dependent features, supporting reliable planning despite visual changes.\\

\noindent \textbf{Comparison with DR.}
Domain randomization preserves performance under SC and C, with $0.0\%$ degradation in both conditions. However, degradation increases to $22.0\%$ under LCG, an aggressive
background change not covered by the randomized training distribution. Our model exhibits $0.0\%$ degradation under this condition, retaining its native-background performance.\\

\noindent \textbf{Moving distractor.} Under the moving distractor, DINO-WM exhibits a small relative degradation of $2.5\%$, while domain randomization preserves its native-background performance. Our model achieves $-5.1\%$ degradation, corresponding to a small
observed improvement over its NC baseline.

Table~\ref{tab:visual_robustness} evaluates visual robustness on PushT in a data-scarce setting, with both models trained on only $1,000$ samples. DINO-Bisim achieves a mean relative degradation of $7.8\%$ across the five visual shifts, compared to $47.6\%$ for DINO-WM. The advantage is larger under aggressive background changes: degradation under LC and LCG is $11.0\%$ for
DINO-Bisim, compared to $71.0\%$ and $50.0\%$,
respectively, for DINO-WM. Under moving distractors,
DINO-Bisim also exhibits lower degradation
($17.0\%$ versus $75.0\%$).

Despite its lower native success rate ($0.36$ versus
$0.48$), DINO-Bisim achieves higher average absolute
success across the five shifted conditions.
These results support improved robustness under visual
shifts when training data are limited.

\begin{table}[t]
\centering
\caption{
\textbf{PushT under data scarcity.}
Both models are trained on $1{,}000$ samples.
We report absolute success rate under NC ($\uparrow$)
and relative degradation (\%, $\downarrow$) under
five visual shifts. Avg.\ denotes mean relative
degradation across the shifted conditions.
}
\label{tab:visual_robustness}
\small
\setlength{\tabcolsep}{3.2pt}
\resizebox{\linewidth}{!}{%
\begin{tabular}{lccccccc}
\toprule
\textbf{Model}
& \shortstack{\textbf{NC}\\(Abs.)}
& \shortstack{\textbf{SC}\\(Rel.\%)}
& \shortstack{\textbf{C}\\(Rel.\%)}
& \shortstack{\textbf{LC}\\(Rel.\%)}
& \shortstack{\textbf{LCG}\\(Rel.\%)}
& \shortstack{\textbf{D}\\(Rel.\%)}
& \shortstack{\textbf{Avg.}\\(Rel.\%)} \\
\midrule
DINO-WM
& 0.48 & 21.0 & 21.0 & 71.0 & 50.0 & 75.0
& 47.6 \\
DINO-Bisim
& 0.36 & 0.0 & 0.0
& 11.0 & 11.0 & 17.0
& \textcolor{blue}{\textbf{7.8}} \\
\bottomrule
\end{tabular}%
}
\end{table}

\subsection{On Different Pretrained Visual Encoders}

Table~\ref{tab:2ndmain_results} compares native-background success rates and relative degradation across visual encoders. End-to-end training (i.e., training a visual encoder from scratch) yields $29.4\%$ mean degradation, compared to $-3.1\%$ with DINOv2, suggesting that pretrained features support robustness when combined with our bisimulation encoder. SimDINOv2 and iBOT also retain their native-background performance, with mean degradation of $3.0\%$ and $0.0\%$, respectively. Thus, robustness extends across pretrained encoders.

\section{Discussions and Future Work} \label{sec:discussion}

We show that DINO-WM is not robust under background changes and distractors. By enforcing control-relevant invariance via bisimulation on top of pretrained visual features, we obtain latent embeddings $10\times$ lower-dimensional than that of DINO-WM that are robust for planning under visual shifts, without reward supervision.

Our work has two main limitations: it relies on informative pretrained visual features, as suggested by reduced robustness without pretraining (i.e., End-to-End in Table~\ref{tab:2ndmain_results}), and its reward-free bisimulation analysis yields a horizon-dependent bound that is looser than its reward-aware counterpart.  Future work includes integrating the bisimulation encoder into gradient-based planning \citep{parthasarathy2025closing}, adapting representations online to new visual conditions, extending the analysis to reward-aware bisimulation, and evaluating more complex manipulation and long-horizon tasks.

\section{Acknowledgments} 
Leonardo F. Toso is funded by the Center for AI and Responsible Financial Innovation (CAIRFI) Fellowship and the Columbia Presidential Fellowship. James Anderson is partially funded by NSF grants EECS 2144634 and CNS 2535097 and the Center of AI Technology (CAIT) in collaboration with Amazon.

\bibliographystyle{apalike}
\bibliography{references}

\newpage
\appendix

\section{Appendix}   

This appendix provides additional details that complement the main text. We expand on the related work discussed in the introduction (Section \ref{sec:intro}) below.  We include specifics on the architecture and implementation of the proposed bisimulation encoder and transition model, further discussion of the PCA-based variance-covariance regularization, and extended descriptions of the experimental setup, including task configurations, domain randomization, and planning details. We also provide the proofs for the generalization bound in Section \ref{sec:generalization}. We begin with a detailed discussions on relevant existing works.

\vspace{-0.3cm}
\subsection{Related Work}\label{apx:related work}

\noindent \textbf{Joint embedding predictive architectures (JEPAs).} Joint embedding predictive architectures learn feature representations by predicting future latent states rather than reconstructing observations \citep{grill2020bootstrap, caron2021emerging, lecun2022path, assran2023self, sobal2025learning, balestriero2025lejepa,maes2026leworldmodel}. Most relevant to our work is DINO-WM \citep{zhoudino} that learns latent dynamics directly in the space of pretrained DINOv2 visual features \citep{oquab2023dinov2}, avoiding pixel-level reconstruction. \textcolor{black}{Other recent joint embedding predictive world models include PLDM \citep{sobal2025learning}, which trains the encoder from scratch with an offline reward-free objective, and EB-JEPA \citep{terver2026lightweight}, which leverages energy-based objectives. In contrast to these approaches, our method operates on top of frozen pretrained visual features, leveraging their semantic richness without retraining them.} However, latent prediction in DINO-WM remains sensitive to slow, control-irrelevant features. We address this limitation by augmenting the predictive objective with a bisimulation encoder to enforce control-relevant invariances.\\

\noindent \textbf{Bisimulation metrics and invariant representations.} Bisimulation provides a principled way of behavioral equivalence in MDPs by quantifying approximate equivalence through reward discrepancies and mismatches in future state distributions for different states under the same policy \citep{ferns2011bisimulation, abate2013approximation}. Most relevant to our work is BS-MPC in \citep{shimizubisimulation}, which incorporates a bisimulation metric loss for model predictive control \cite{kouvaritakis2016model, lenz2015deepmpc, williams2017information}. In contrast to \citep{shimizubisimulation}, we leverage bisimulation metrics to ``filter out" control-irrelevant information from a pretrained feature space (e.g., that of DINOv2, SimDINOv2, iBOT). Moreover, we jointly train our prediction model without relying on reward prediction. \textcolor{black}{Distance-aware representation learning methods such as HILP \citep{park2024foundation} share the goal of learning metrically structured latent spaces, but are typically applied in the context of goal-conditioned policy learning rather than world-model-based planning on top of frozen visual features.}\\

\noindent  \textbf{Model-based control with latent world models.} Model-based RL with world models enable planning by rolling out learned dynamics in a compact latent space, which has been effective for high-dimensional tasks \citep{ha2018world, hafner2019learning, hafner2019dream}. In robotics, latent dynamics models are commonly paired with sampling-based optimizers for action selection, including the cross-entropy method (CEM), which is widely used for planning with learned world models \citep{rubinstein2004cross, zhoudino}. Our control stack follows this line, we plan with CEM with the learned latent dynamics, and focus on the representation learning side, i.e., how to learn a latent space that is robust to slow features so that planning remains reliable under visual distribution shifts in test time. \textcolor{black}{Approaches for improving robustness to visual distractors in model-based RL include HRSSM \citep{sun2024learning} and JEDI \citep{lim2025jedi}, which address distractor robustness via stochastic state models and diffusion-based planning respectively. These methods typically learn representations end-to-end from scratch, while our approach adapts frozen pretrained visual features and focuses on suppressing control-irrelevant directions via bisimulation.}

\addcontentsline{toc}{section}{Appendix}

\subsection{Model Architecture} \label{appendix:archictecture}

This section provides additional details on the architecture of the proposed joint-embedding predictive world model with bisimulation encoder.

\subsubsection{Overview}

The DINO-Bisim model takes visual observations, proprioceptive inputs, and actions at the current timestep and predicts the next bisimulation latent state, along with an optional reward prediction (see Figure \ref{fig:dino_bisim_overview}). Visual observations are encoded using a pretrained DINOv2 ViT-S/14 visual encoder \cite{oquab2023dinov2}, followed by a bisimulation encoder implemented as an MLP to produce the latent state. Proprioceptive inputs and actions are embedded using lightweight 1D convolutional encoders. The resulting embeddings are concatenated and processed by a 6-layer transformer to predict the next bisimulation latent state. A separate reward MLP predicts the reward from the latent state, however, this component is not used in our experiments, as discussed in Section~\ref{sec:experiments} and throughout the main text.

\begin{figure}[t]
    \centering
    \includegraphics[width=1\linewidth]{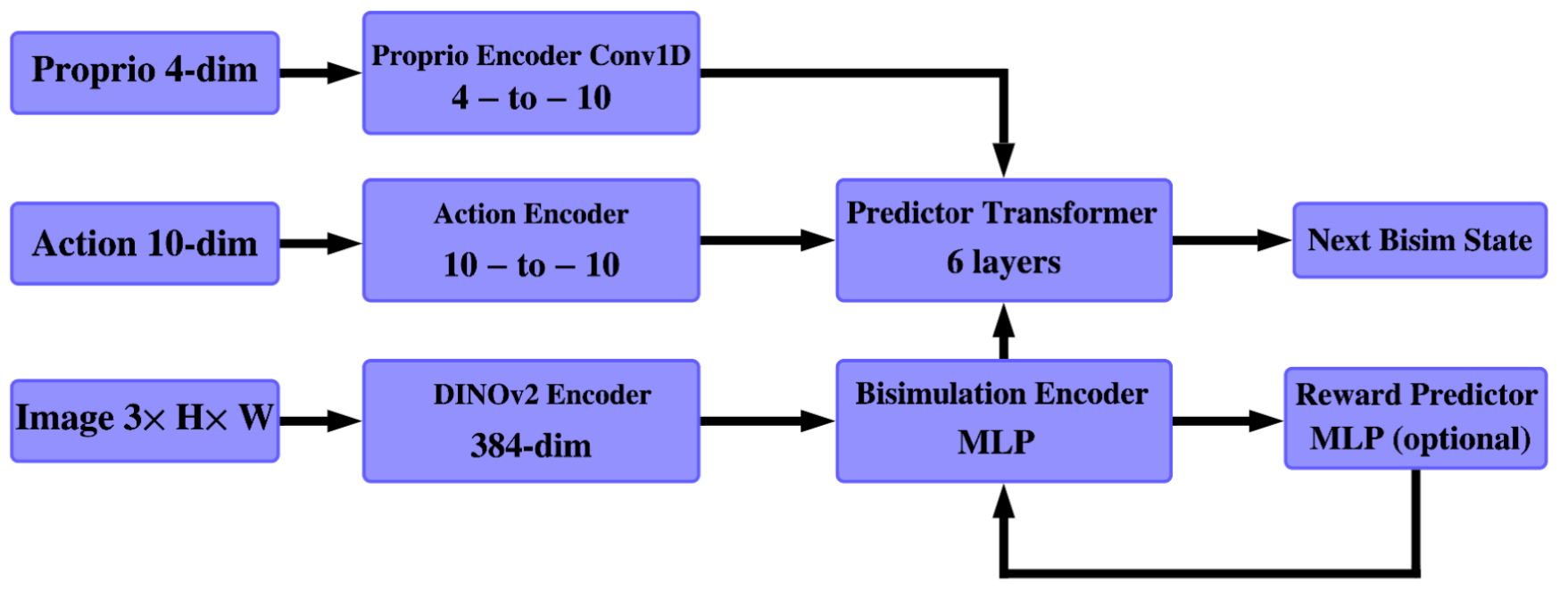}
    \caption{Overview of the DINO-Bisim architecture.}
    \label{fig:dino_bisim_overview}
\end{figure}

\subsubsection{Bisimulation Memory Buffer (Replay Buffer)}
Recalling the definition of the bisimulation metric introduced in Section \ref{subsec:bisimulation} and in \cite{castro2020scalable}, the corresponding loss involves computing differences between selected pairs of states. Since a single training iteration can only leverage the current batch of $20$ states, we introduce a bisimulation memory buffer that stores up to $1000$ previously generated states, also following \cite{zhanglearning}. This buffer allows comparisons with a larger and more diverse set of past states, thereby improving the stability and effectiveness of the bisimulation representation learning.

\subsubsection{Patch design in Bisimulation Model} \label{appendix:patch_design}

\noindent \textbf{Why Patch-based Design?} The bisimulation model adopts a patch-based design, in which visual inputs are partitioned into spatial patches and each patch is encoded independently. This design builds upon the bisimulation encoder in \cite{shimizubisimulation} and naturally aligns with the output format of DINOv2 and other patch-based visual encoders. Beyond architectural compatibility, the patch-based formulation substantially reduces the number of trainable parameters, leading to a more compact and efficient model while empirically improving performance.

More concretely, an example architecture using the hyperparameters in Table \ref{tab:bisimulation_params} is illustrated in Figure \ref{fig:bisim_architecture}. Rather than flattening a feature map of size $14 \times 14 \times 384$ into a $75264$-dimensional vector as input to the first layer, the patch-based bisimulation encoder performs $196$ parallel forward passes using a shared MLP whose input dimension is $384$. As a result, the number of parameters in the first layer is reduced from $(75,264 \times 256) + 256 = 19,267,840$ to $(384 \times 256) + 256 = 98,560$, where $256$ is the hidden layer dimension, with only negligible additional overhead from the spatial embedding parameters.

\vspace{0.2cm}

\noindent \textbf{Latent Space Dimensionality Comparison.} DINO-WM operates on patch-based visual representations produced by DINOv2, consisting of $196$ patches with $384$-dimensional embeddings each, resulting in a latent representation of size $196 \times 384$. In contrast, our patch-based bisimulation encoder preserves the same spatial structure but maps each patch to a low-dimensional latent embedding. This yields a representation of size $196 \times 32$. Consequently, our model operates in a latent space that is more than $10\times$ smaller than that of DINO-WM, while maintaining improved robustness to slow, task-irrelevant features.

\subsubsection{Hyperparameters} \label{appendix:hyperparamters}

We now summarizes the hyperparameters used in our experiments in Tables \ref{tab:dino_bsmpc_hparams} and \ref{tab:bisimulation_params}.

\begin{table}[t]
\centering
\begin{minipage}[t]{0.48\columnwidth}
\centering
\caption{Shared hyperparameters for DINO-Bisim training.}
\vspace{0.2cm}
\label{tab:dino_bsmpc_hparams}
\begin{tabular}{lc}
\toprule
\textbf{Name} & \textbf{Value} \\
\midrule
Image size        & 224   \\
Optimizer         & AdamW \\
Decoder learning rate        & $3\mathrm{e}{-4}$ \\
Predictor learning rate     & $1\mathrm{e}{-5}$ \\
Action encoder learning rate & $1\mathrm{e}{-4}$ \\
Bisimulation learning rate         & $5\mathrm{e}{-7}$ \\
Action embedding dim.    & 10    \\
Proprio embedding dim.   & 10    \\
Epochs            & 50   \\
Batch size        & 20    \\
\bottomrule
\end{tabular}
\end{minipage}
\hfill
\begin{minipage}[t]{0.48\columnwidth}
\centering
\caption{Bisimulation encoder hyperparameters.}
\vspace{0.2cm}
\label{tab:bisimulation_params}
\begin{tabular}{lc}
\toprule
\textbf{Name} & \textbf{Value} \\
\midrule
Bisimulation memory buffer size & 1000 \\
Bisimulation state comparison size   & 200  \\
Action dim.      & 10   \\
Num patches     & 196  \\
Latent dim.      & 32    \\
Patch dim.       & 32    \\
Patch embedding dim.  & 384  \\
\bottomrule
\end{tabular}
\end{minipage}
\end{table}

We list below the main hyperparameters used by the bisimulation encoder and transition model. These parameters control the dimensionality of the latent representation, the size of the comparison set used for bisimulation representation learning, and the dimensionality of visual and action embeddings.
\vspace{0.2cm}

\noindent \textbf{Bisimulation Memory Buffer Size and Comparison Size.}
The capacity of the bisimulation memory buffer and the number of stored latent states used for pairwise comparisons when computing the bisimulation loss.
\vspace{0.2cm}

\noindent \textbf{Action Dimension (\texttt{action\_dim}).}
The dimensionality of the encoded action representation, used as input to the reward predictor that is not in use for our experiments.

\noindent \textbf{Number of Patches (\texttt{num\_patches}).}
The number of spatial patches output by the bisimulation encoder, matching the patch decomposition of the visual pretrained encoder.
\vspace{0.2cm}

\noindent \textbf{Latent Dimension (\texttt{latent\_dim}).}
The dimensionality of the bisimulation latent space for each patch.
\vspace{0.2cm}

\noindent \textbf{Patch Dimension (\texttt{patch\_dim}).}
The dimensionality of each patch embedding produced by the bisimulation encoder; in our implementation, this is set equal to \texttt{latent\_dim}.
\vspace{0.2cm}

\noindent \textbf{Patch Embedding Dimension (\texttt{patch\_emb\_dim}).}
The dimensionality of the input patch embeddings provided by the pretrained visual encoder, matching the output dimension of DINOv2.

\subsection{PCA-based Variance Covariance Regularization} \label{appendix:PCA-loss}

We now discuss the implementation of our PCA-based variance covariance regularization.

Let 
$
W_c \triangleq \begin{bmatrix}
(w_1-\bar w)^\top &
\ldots &
(w_B-\bar w)^\top
\end{bmatrix},
$
be the matrix with centered embeddings in the batch of size $B$, where $\bar w$ is the average. We define the empirical covariance $\Sigma \triangleq \frac{W_c^\top W_c}{B-1}$, and further compute its eigendecomposition as $\Sigma = U \Lambda U^\top$, where $\Lambda = \mathrm{diag}(\lambda_1,\ldots,\lambda_{d_w})$ \textcolor{black}{with eigenvalues sorted in descending order ($\lambda_1 \geq \lambda_2 \geq \cdots \geq \lambda_{d_w} \geq 0$; non-negativity is guaranteed since $\Sigma$ is a positive semi-definite empirical covariance matrix)}. We extract the PCA coordinates $\tilde w_b = U^\top (w_b-\bar w)$, and let $\tilde W_c$ be the batch matrix with rows $\tilde w_b^\top$.

To construct the PCA-based VC-Reg loss, we only enforce the variance floor on the tail PCs, namely indices $i \in \{i_0,\ldots,d_w\}$, where $i_0$ is a cutoff index (e.g., $i_0 = 2$). Let $\kappa_i \triangleq \sigma_{\min} - \sqrt{\lambda_i+\epsilon}$ and write the variance term
$$
\mathcal{L}_{\mathrm{pca\text{-}var}}(\eta) \triangleq \frac{1}{d_w-i_0+1}
\sum_{i=i_0}^{d_w} \left[\max\left(0, \kappa_i \right)\right]^2.
$$
where $\epsilon>0$ is a small constant for numerical stability \textcolor{black}{(ensuring the square root is well-defined even for near-zero eigenvalues)} and $\sigma_{\min}$ denotes a minimum standard deviation. On the other hand, we regularize cross-covariances as 
$$
\mathcal{L}_{\mathrm{pca\text{-}cov}}(\eta) \triangleq \frac{1}{d_w(d_w-1)}
\sum_{p\neq q} \left(\tilde \Sigma_{pq}\right)^2,
\; \tilde \Sigma \triangleq \frac{\tilde W_c^\top \tilde W_c}{B-1}, 
$$
where $\tilde \Sigma_{pq}$ denotes the $(p,q)$ entry of the covariance in PCs coordinates, namely, the covariance matrix $\tilde{\Sigma}$. Therefore, we combine the two PCA-based terms
\begin{align*}
\mathcal{L}_{\mathrm{pca\text{-}vc}}(\eta) \triangleq
\alpha_{\mathrm{v}}\,\mathcal{L}_{\mathrm{pca\text{-}var}}(\eta)+ \alpha_{\mathrm{c}}\,\mathcal{L}_{\mathrm{pca\text{-}cov}}(\eta),
\end{align*}
with weights $\alpha_{\mathrm{v}},\alpha_{\mathrm{c}}>0$, to the predictive loss \eqref{eq:overall_prelim_obj}.\\

\noindent\textbf{Setup.} Our experiments use DINOv2, SimDINOv2 and iBOT as pretrained visual encoder, which encode background appearance along the dominant principal components (see Section~\ref{subs:PCA-loss}). As noted in the DINOv2 paper \cite{oquab2023dinov2}, background information can be attenuated by thresholding the first principal component. Consistent with this observation and our illustration in Figures \ref{fig:pca-pointmaze} and \ref{fig:pca_background}, we make the preliminary assessment that task-irrelevant appearance factors predominantly align with the first principal component direction in the latent feature space.\\

\noindent\textbf{Method.} The bisimulation encoder is initialized with random parameters and is first trained using the standard VICReg objective for a small number of epochs to stabilize representation learning. We then perform a one-time principal component analysis to identify the dominant direction associated with task-irrelevant variation. After a number of warm-up epochs, we replace the standard VICReg loss with the proposed PCA-based VICReg objective, which helps to suppresses variance along the identified principal component directions.

An example training run with a vectorized bisimulation latent representation of dimension $196 \times 32$ (corresponding to the number of patches and per-patch latent dimensionality), where PCA-based VICReg is enabled after $k=50$ epochs, is shown in Figure~\ref{fig:pca_loss_demo}.\\

\noindent\textbf{Result.} While this staged regularization approach improves the stability of invariant representation learning, we hypothesize that a tighter integration of PCA-based regularization throughout training and planning, rather than a discrete loss swap, could further enhance invariance. Exploring such joint objectives is an interesting direction for future work.

\begin{figure}
    \centering
    \includegraphics[width=1\linewidth]{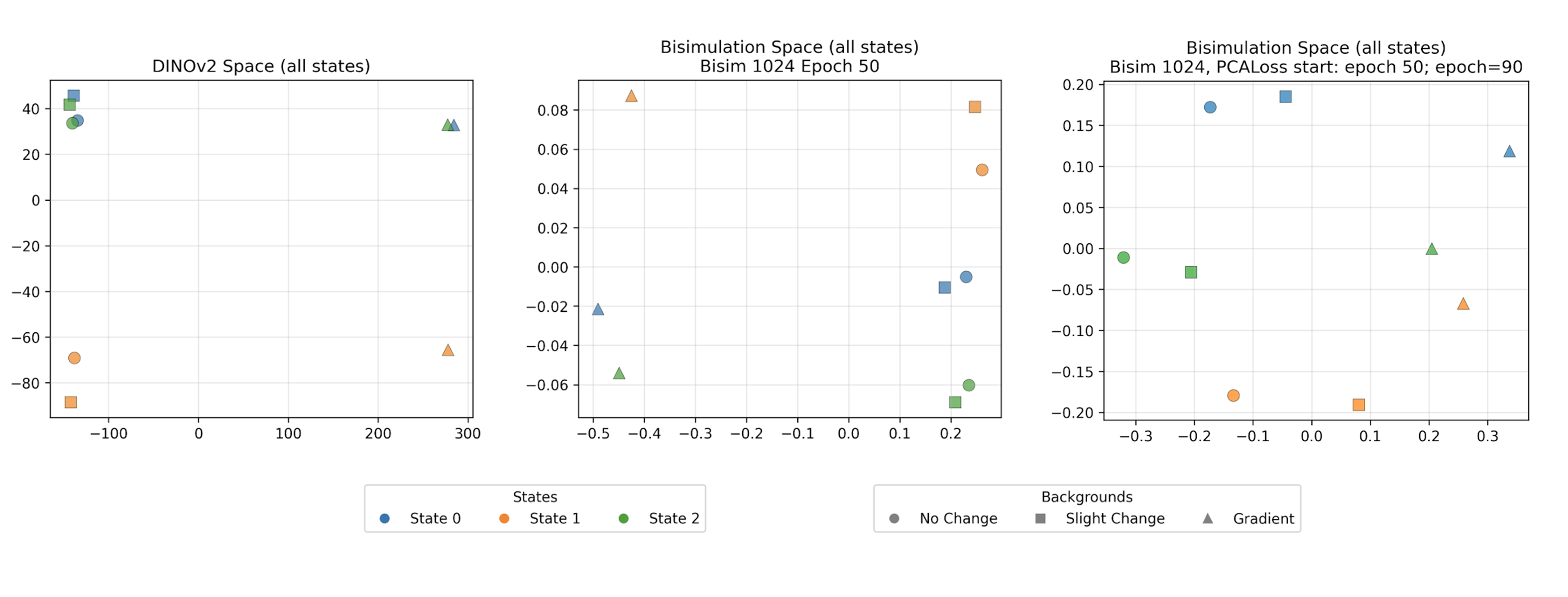}
    \vspace{-0.4cm}
    \caption{First two principal components (PCs) of three distinct states (colors) observed under three different backgrounds (marker shapes) for the PointMaze task. The goal of the bisimulation encoder is to map identical states across backgrounds to nearby latent representations. \textbf{Left.} DINOv2 embeddings $z_t = f_\theta(o_t)$ show strong separation along the first PC for identical states under different backgrounds, indicating that background variation dominates the leading principal direction. \textbf{Middle.} Bisimulation embeddings $w_t = h_\eta(z_t)$ after 50 epochs with standard VICReg still exhibit separation along the first PC, preserving background-dependent variation. \textbf{Right.} Bisimulation embeddings after 90 epochs with PCA-based VICReg (activated at epoch 50) show improved clustering of identical states across backgrounds, consistent with suppression of task-irrelevant variation.}
    \label{fig:pca_loss_demo}
\end{figure}

\begin{figure}[h]
    \centering
    \includegraphics[width=1\linewidth]{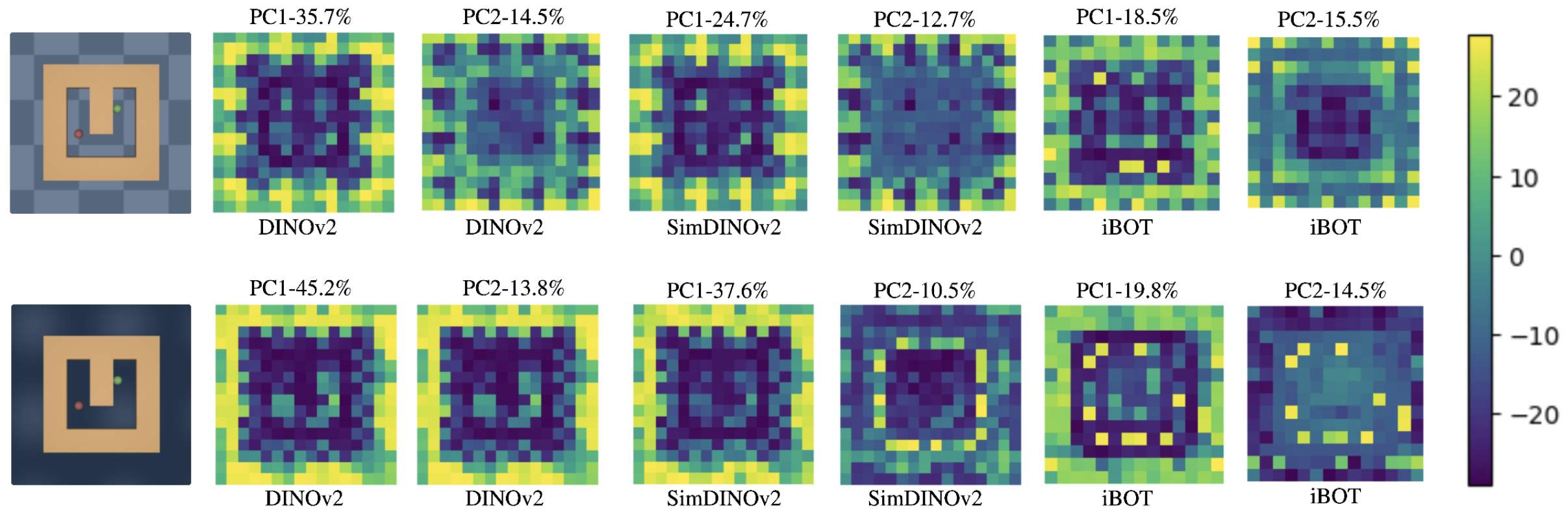}
    \vspace{-0.4cm}
    \caption{First twp principal components (PC1 and PC2) of latent embeddings produced by DINOv2, SimDINOv2, and iBOT for a PointMaze observation. In all cases, PC1 captures a large percentage of the total variance that predominantly encodes background and layout information rather than control-relevant features.}
    \label{fig:pca_background}
\end{figure}

\subsection{Additional Information on the Navigation Task} 

\paragraph{Task Description.} This section provides additional details on the simple navigation task used in our experiments, complementing the high-level description in the main text. PointMaze \cite{fu2020d4rl} is a force-actuated 2-DoF navigation environment in which a ball moves in the Cartesian $x$--$y$ plane with the objective of reaching a target location. The agent observes proprioceptive state information and evolves under continuous dynamics that capture velocity, acceleration, and inertia. For a fair comparison, we train all models on $2000$ fully random trajectories, matching the data setup in \cite{zhoudino}. See Figure \ref{fig:PointMazeTask} for a sequence of frames that illustrates the planning from an initial state to a goal state.

\label{appendix:tasks}
\begin{figure}[h]
    \centering
    \includegraphics[width=1\linewidth]{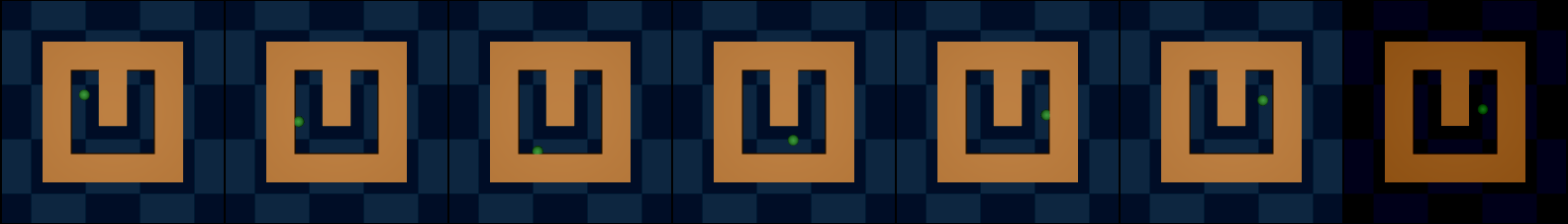}
    \vspace{-0.4cm}
    \caption{PointMaze task, from leftmost the initial state to the rightmost the goal state}
    \label{fig:PointMazeTask}
\end{figure}

Fig.~\ref{fig:rollout-pm-wm-fail} and Fig.~\ref{fig:rollout-pm-bisim-succ} and  shows representative MPC rollouts on PointMaze  under six test-time visual conditions (columns) over six replanning steps (rows). DINO-WM (top/bottom-left pairs) fails under strong appearance shift while DINO-Bisim (paired success rows) reaches the goal from the same start and target. For this additional illustration we have also considered different background changes and more ``aggressive'' distractors compared to those in the discussed in the main body of the paper.

\subsection{Additional Information on the  Manipulation Task}
\begin{figure}[t]
    \centering
    \includegraphics[width=1\linewidth]{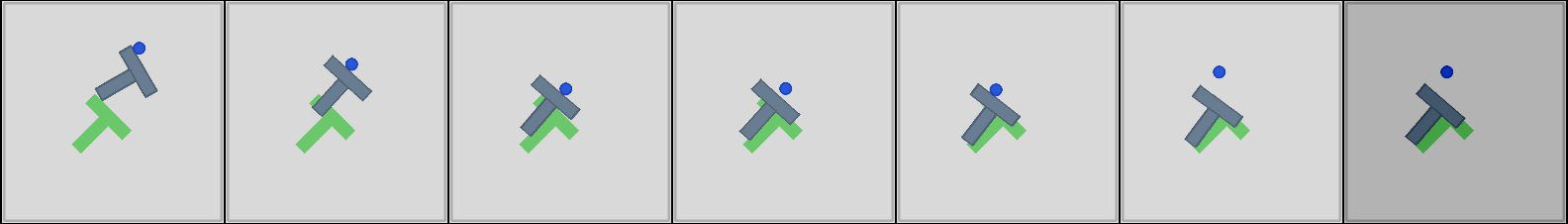}
    \caption{Push T Task, from leftmost the start state to the rightmost the goal state}
    \label{fig:PushT}
\end{figure}

\textbf{Task Description.} PushT environment consists of a pusher agent interacting with a T-shaped object. Starting from a randomly initialized configuration, the agent must drive both itself and the T-block to a known feasible target state, requiring both entities to reach their respective target positions (see Figure \ref{fig:PushT}). Unlike prior setups, the fixed green T does not indicate the target pose of the object and instead serves solely as a visual anchor for reference. We train the model on $18500$ samples exact same with \cite{zhoudino}, for the best performance comparison.\\

\noindent \textbf{Raw success scores.} Fig.~\ref{fig:rollout-pt-wm-fail} and Fig.~\ref{fig:rollout-pt-bisim-succ} and  shows representative MPC rollouts on PointMaze  under six test-time visual conditions (columns) over six replanning steps (rows). DINO-WM (top/bottom-left pairs) fails under strong appearance shift while DINO-Bisim (paired success rows) reaches the goal from the same start and target.

\begin{figure}[t]
  \centering
  \includegraphics[width=\textwidth]{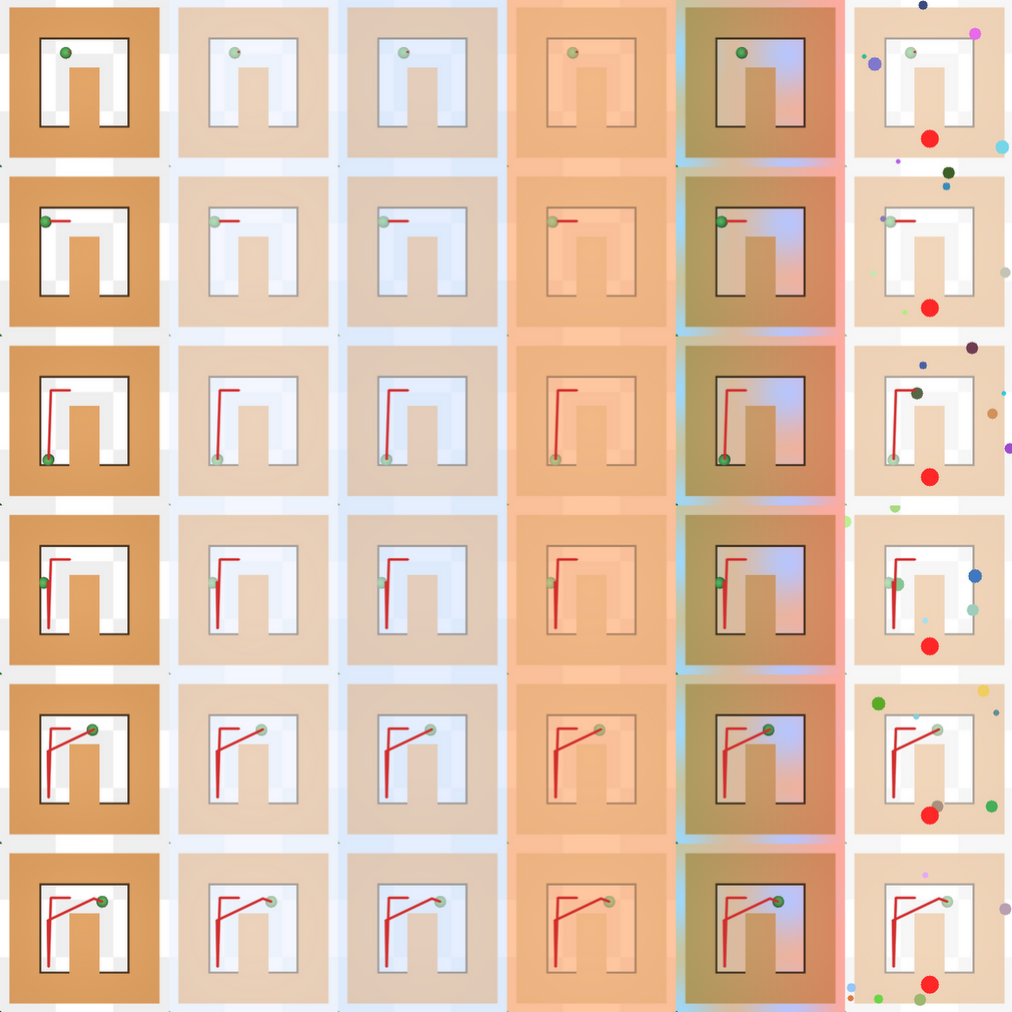}
  \caption{%
    \textbf{PointMaze planning failure (DINO-WM).}
    Six MPC steps (rows) under six visual conditions (columns: NC, SC, C, LC, LCG, D).
    Markers: start (orange), goal (green); red line: rollout path.
    Same init/goal as Fig.~\ref{fig:rollout-pm-bisim-succ}. DINO-WM fails to reach the goal, especially under LC/LCG.
  }
  \label{fig:rollout-pm-wm-fail}
\end{figure}

\begin{figure}[t]
  \centering
  \includegraphics[width=0.9\textwidth]{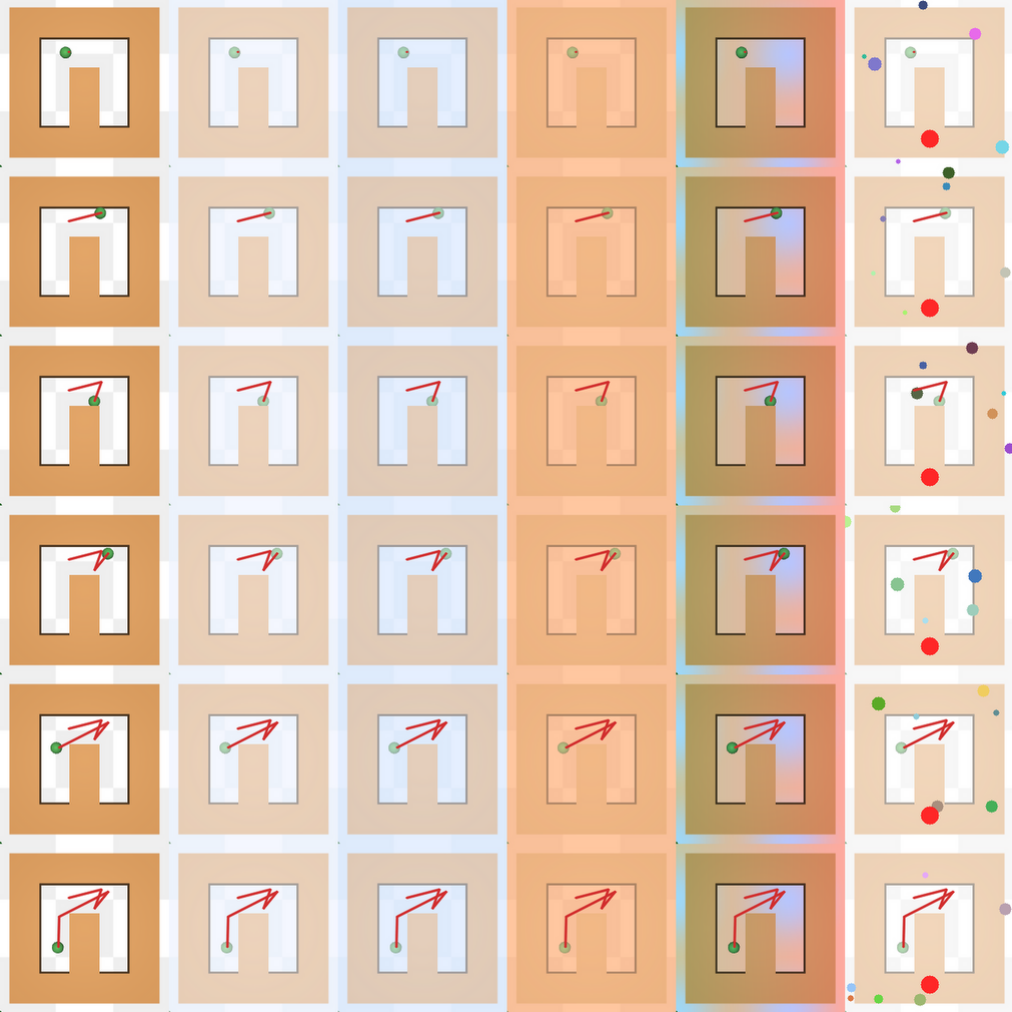}
  \caption{%
    \textbf{PointMaze planning success (DINO-Bisim).}
    Same layout and shared episode as Fig.~\ref{fig:rollout-pm-wm-fail}.
    Bisimulation-aligned representations yield a goal-reaching trajectory across all six appearance conditions.
  }
  \label{fig:rollout-pm-bisim-succ}
\end{figure}

\begin{figure}[t]
  \centering
  \includegraphics[width=0.9\textwidth]{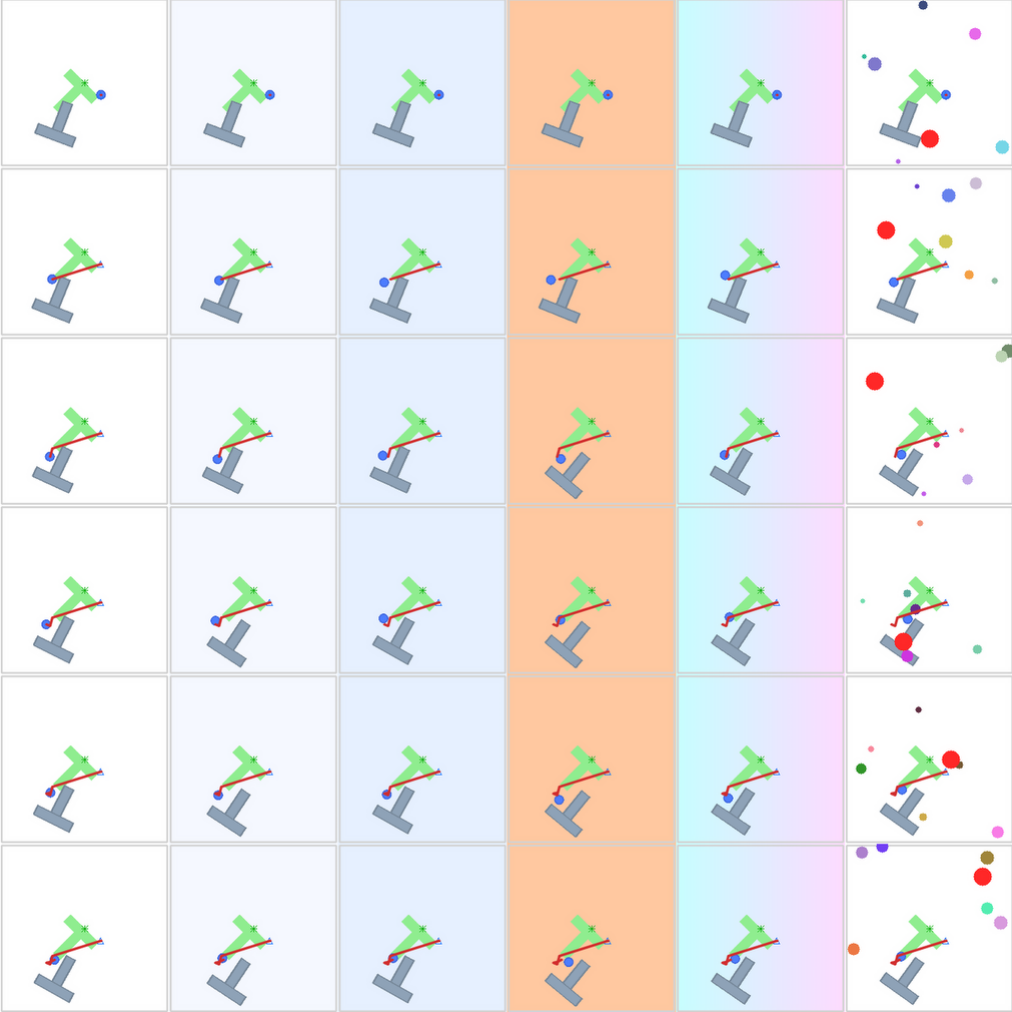}
  \caption{%
    \textbf{PushT planning failure (DINO-WM).}
    Six MPC steps $\times$ six visual conditions; start/goal markers and executed path overlaid.
    Matched init/goal with Fig.~\ref{fig:rollout-pt-bisim-succ}.
    WM misaligns the block with the target pose under shifted backgrounds.
  }
  \label{fig:rollout-pt-wm-fail}
\end{figure}

\begin{figure}[t]
  \centering
  \includegraphics[width=0.9\textwidth]{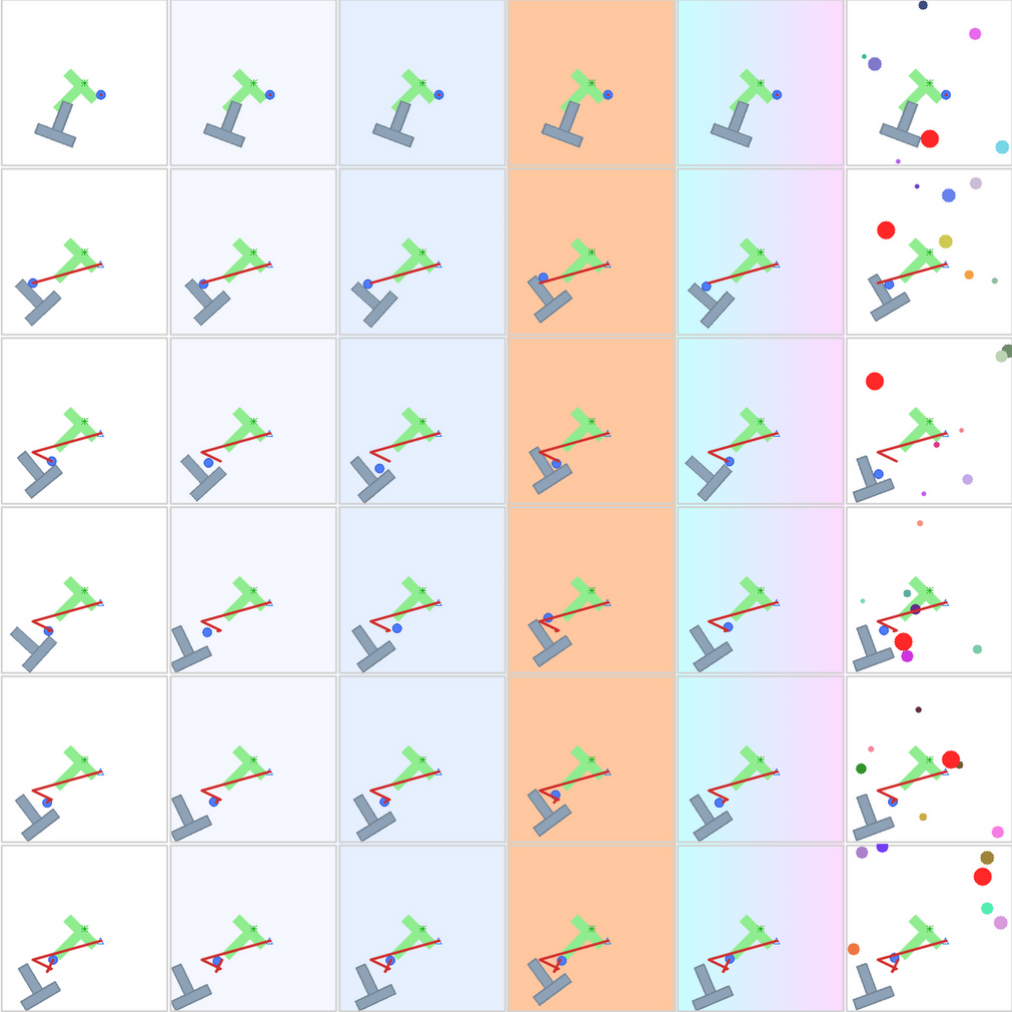}
  \caption{%
    \textbf{PushT planning success (DINO-Bisim).}
    Same episode as Fig.~\ref{fig:rollout-pt-wm-fail}.
    DINO-Bisim completes the push despite lighting/color perturbations at test time.
  }
  \label{fig:rollout-pt-bisim-succ}
\end{figure}
\clearpage

\subsection{Planning}

We optimize the action sequence using CEM. At each step, CEM samples multiple candidate action sequences, rolls them out using the learned model, selects the sequences with lowest cost \eqref{goal_reaching_cost}, and updates the sampling distribution accordingly. The first action of the optimized sequence is played, and the process repeats with a receding horizon.

Although gradient-based planning methods have recently been proposed and validated in \citep{parthasarathy2025closing} for DINO-WM, we focus on CEM for a fair comparison with the planning approach used in DINO-WM \cite{zhoudino}.

\subsection{Domain Randomization for the PointMaze Task} \label{appendix:domain_randn}

To evaluate robustness through data augmentation, we apply domain randomization to the PointMaze task by modifying the visual appearance of the environment during training. In particular, $40\%$ of the training data are rendered with randomized visual backgrounds, while the remaining trajectories retain the default appearance. The randomization includes background color changes and gradient-based background modifications, as illustrated in Figure \ref{fig:DR_examples}.

To avoid introducing spurious temporal artifacts that could make the transition dynamics learning harder, the randomized background is kept fixed within each trajectory but varies across different trajectories. Importantly, these visual modifications do not alter the underlying environment dynamics, agent behavior, or goal locations, ensuring that domain randomization affects only task-irrelevant visual features.

\begin{figure}[t]
    \centering
    \includegraphics[width=1\linewidth]{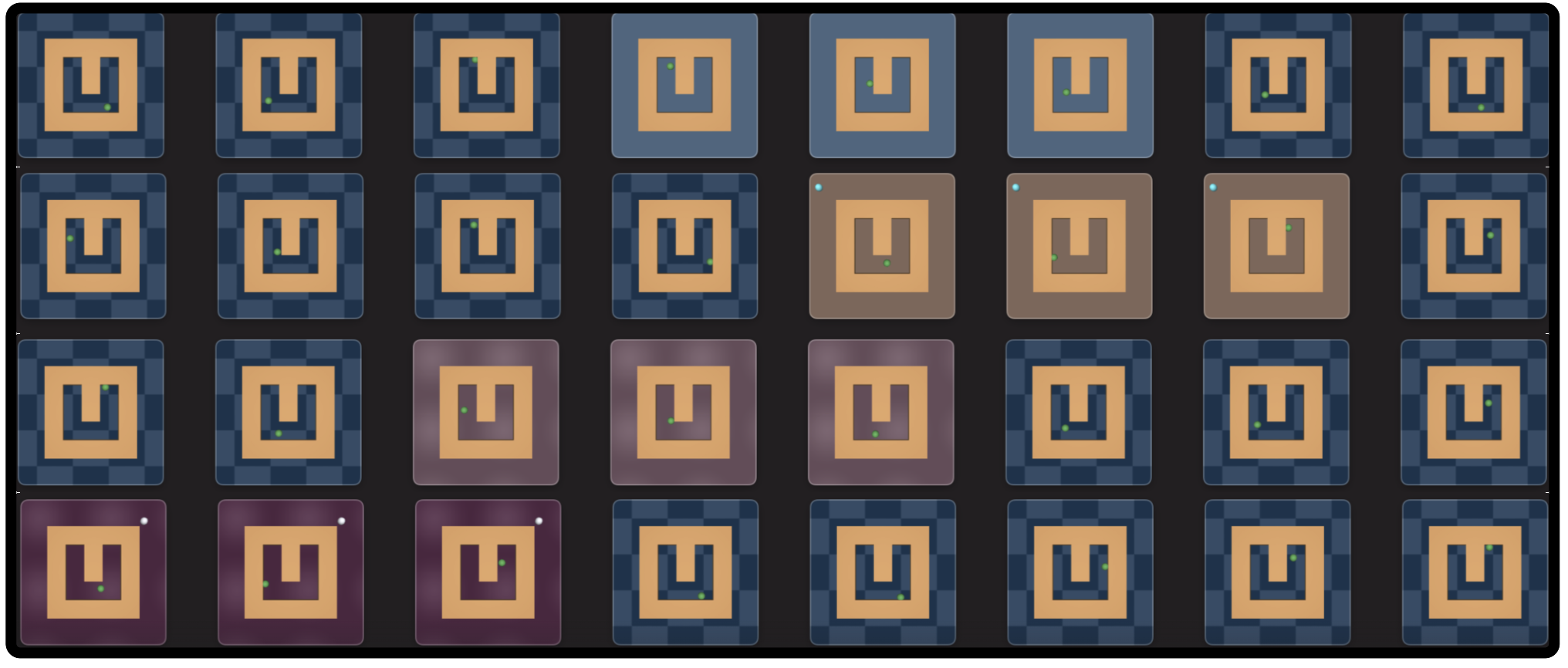}
    \vspace{-0.4cm}
    \caption{Examples of domain-randomized visual observations for the PointMaze task. During training, background appearance is randomized across trajectories using color shifts, gradient backgrounds, and fixed distractors (i.e., the dot outside the maze) while the underlying maze layout and transition dynamics remain unchanged.}
    \label{fig:DR_examples}
\end{figure}

\subsection{Reward-aware Bisimulation Encoder: Training Metrics}\label{appendix:reward_aware}

Let us now analyze the training dynamics of the reward-aware bisimulation encoder to assess the contribution of the reward prediction term relative to the transition-based objectives. Following prior work on bisimulation-based MPC \cite{shimizubisimulation}, we include an auxiliary reward predictor during training, although it is not used for training nor planning in our experiments.

Figure \ref{fig:training_metrics} reports the evolution of the main loss components throughout training. We observe that the reward prediction loss rapidly converges to near zero early in training and remains negligible afterwards. This indicates that, for the navigation task considered, reward prediction provides little additional learning signal beyond what is already captured by the transition dynamics. As a result, the reward term has a negligible influence on the learned latent representation.

\begin{figure}
    \centering
    \includegraphics[width=0.9\textwidth]{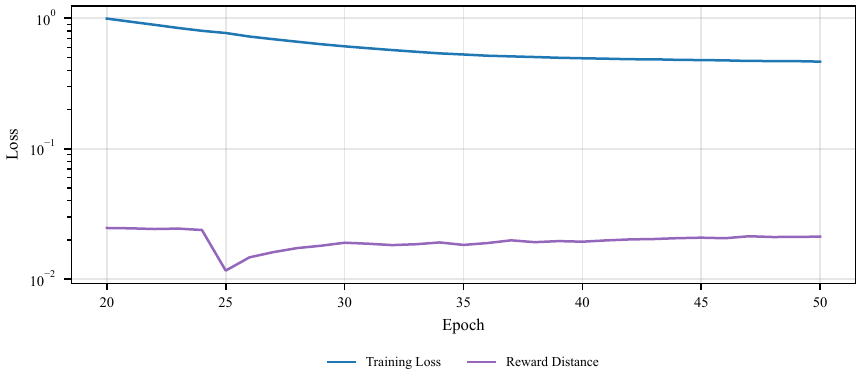}
    \vspace{-0.2cm}
    \caption{Training losses value versus the epochs. Ground-truth reward distance is shown only as a diagnostic and is not used during training.}
    \label{fig:training_metrics}
    \vspace{-0.4cm}
\end{figure}

On the other hand, the transition losses and representation regularization terms (including the VICReg and bisimulation objectives) remain active throughout training and dominate the optimization process. These components are responsible for making the latent space to preserve control-relevant dynamics while suppressing task-irrelevant visual features. This empirical observation supports the design choice adopted in the main experiments, where we drop the reward prediction term and rely only on transition-based bisimulation to enforce invariance.

\subsection{Moving Distractors}

To evaluate robustness to dynamic but task-irrelevant visual variations, we consider two moving distractors in the PointMaze task. As illustrated in Figure \ref{fig:moving_distractors_appendix},  we augment the environment with colored distractor objects (yellow and magenta dots) that move continuously within the maze while remaining decoupled from the agent dynamics and task objectives.

The distractors introduce persistent motion that can dominate DINO-WM latent dynamics, despite being irrelevant for control.  On the other hand our bisimulation-based representation is designed to suppress such task-irrelevant motion by enforcing equivalence between states that exhibit identical on-policy transition behavior. As a result, planning performance remains stable in the presence of moving distractors, whereas DINO-WM tend to drop in performance.

\begin{figure}[h]
    \centering
    \includegraphics[width=1\linewidth]{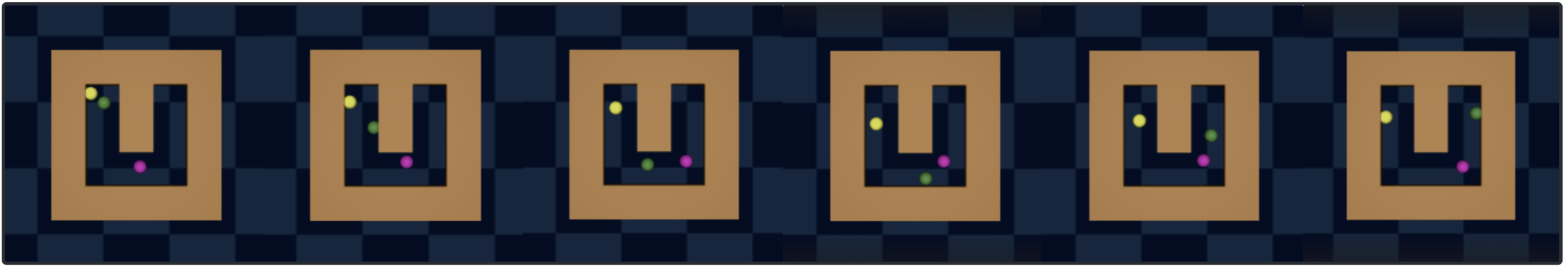}
    \vspace{-0.4cm}
    \caption{Moving distractor setting in PointMaze. Yellow and magenta dots move continuously within the maze, introducing dynamic but task-irrelevant visual variation while preserving the underlying transition dynamics.}
    \label{fig:moving_distractors_appendix}
\end{figure}

\subsection{Ablation (planning horizon)}
\label{appendix:horizon_ablation}

Table~\ref{tab:main_results2} evaluates the effect of the MPC planning horizon on PointMaze under NC, SC, and C. Increasing the horizon from $T=5$ to $T=10$ improves the success rate from $0.78$, $0.80$, and $0.76$, respectively, to $0.88$ in all three conditions. Further benefit is also observed at moderate horizons: $T=25$ achieves the highest mean success rate across the three conditions ($0.92$, compared to $0.78$ at
$T=5$). However, the improvement is not monotonic,
and extending the horizon to $T=50$ or $T=100$
yields a lower mean success rate of approximately
$0.89$. These results suggest that moderate planning horizons improve goal-reaching performance while preserving robustness to the evaluated background changes.
\begin{table}[h]
\centering
\caption{\textbf{PointMaze.} Planning-horizon ablation
for DINO-Bisim. Absolute success rates under the native
background (NC) and background changes (SC and C).
Higher is better.}
\label{tab:main_results2}
\begin{tabular}{lccc}
\toprule
\textbf{Horizon} & \textbf{NC} & \textbf{SC} & \textbf{C} \\
\midrule
5   & 0.78 & 0.80 & 0.76 \\
10  & 0.88 & 0.88 & 0.88 \\
15  & 0.94 & 0.88 & 0.92 \\
20  & 0.88 & 0.92 & 0.90 \\
25  & 0.94 & 0.94 & 0.88 \\
50  & 0.84 & 0.90 & 0.92 \\
100 & 0.86 & 0.92 & 0.88 \\
\bottomrule
\end{tabular}
\vspace{-0.3cm}
\end{table}

\subsection{Reward-Free Generalization Bound: Proof of Theorem \ref{thm:rf_generalization}} \label{appendix:proof_generalization}
We  first provide an auxiliary lemma regarding the 1-Wasserstein distance which we use in the proof of Theorem \ref{thm:rf_generalization}.

\begin{lemma}[Kantorovich-Rubinstein duality \cite{villani2008optimal}] \label{lemma:KR_duality}
Let $ \mu, \text{ and } \nu$ be two probability measures on the metric space $(\Omega,d_\pi)$, then
\begin{align}\label{eq:KR_duality}
W_1(\mu,\nu) = \sup_{\|f\|_{L}\le 1} \left( \mathbb{E}_{x\sim\mu}[f(x)]-
\mathbb{E}_{y\sim\nu}[f(y)]\right),   
\end{align}
where $\|\cdot\|_{L}$ is the Lipschitz norm (i.e., $\|f\|_L \leq \sup_{x\neq y}\frac{|f(x)-f(y)|}{|x-y|}$).
\end{lemma}

We will later invoke the auxiliary result above in the proofs of Theorem~\ref{thm:rf_generalization}.  To make this step explicit, we first show how the Kantorovich-Rubinstein duality implies a 
Lipschitz-based bound on differences of expectations. Let $ \mu \triangleq P_\pi(\cdot\mid z), \text{ and } \nu \triangleq P_\pi(\cdot\mid z'),$ be two probability measures on the metric space $(\Omega,d_\pi)$, and define a scalar-valued function $s(\tilde z)$ on the latent space $z$. Assuming that $s(z)$ is a $L$-Lipschitz with respect to $d_\pi$, i.e.,
$$
|s(u)-s(v)| \leq L d_\pi(u,v), \; \forall u,v\in\Omega.
$$
By Kantorovich-Rubinstein duality (Lemma \ref{lemma:KR_duality}), we have 
\begin{align*}
W_1(\mu,\nu) = \sup_{\|f\|_{L}\le 1} \left( \mathbb{E}_{x\sim \mu}[f(x)] - \mathbb{E}_{y\sim \nu}[f(y)]\right),
\end{align*}
where $\|f\|_{L}\le 1$ denotes that $f$ is $1$-Lipschitz with respect to $d_\pi$. Now define the rescaled function $l(\cdot) \triangleq \frac{1}{L}s(\cdot).$ Then $l$ is $1$-Lipschitz since, for all $u,v\in\Omega$, we have
$$
|l(u)-l(v)| = \frac{1}{L}|s(u)-s(v)| \leq
\frac{1}{L}L_{}d_\pi(u,v) = d_\pi(u,v).
$$

Thus $l$ is feasible for the supremum in \eqref{eq:KR_duality}, implying that
$$
W_1(\mu,\nu) \geq \mathbb{E}_{x\sim \mu}[l(x)] - \mathbb{E}_{y\sim \nu}[l(y)] = \frac{1}{L} \left( \mathbb{E}_{x\sim \mu}[s(x)] - \mathbb{E}_{y\sim Q}[s(y)] \right).
$$

Rearranging the terms, we obtain
\begin{align}\label{eq:KR_lispzhitz}
\mathbb{E}_{x\sim \mu}[s(x)] - \mathbb{E}_{y\sim \nu}[s(y)] \leq
LW_1(\mu,\nu).    
\end{align}

\noindent \textbf{Proof of Theorem \ref{thm:rf_generalization}:}
We begin the proof by first noticing that the goal-reaching cost used for planning is Lipschitz on the bounded set of bisimulation latent states, i.e., $\left\{w:\|w\|_2\leq {H}_w\right\}$. Indeed, for any $w,w'$ with $\|w\|_2,\|w'\|_2,\|w_g\|_2\leq {H}_w$, we have the following expression
\begin{align}
\left|c(w)-c(w')\right| &= \left|\|w-w_g\|_2^2-\|w'-w_g\|_2^2\right| \notag \\
&= \left|\langle w-w',\,w+w'-2w_g\rangle\right| \notag \\
&\leq \|w-w'\|_2 \|w+w'-2w_g\|_2\notag \;\; \text{  (Cauchy–Schwarz inequality)}\\
&\leq 4{H}_w \|w-w'\|_2,
\end{align}
so $c(w)$ is $4{H}_w$-Lipschitz.

We further write the recursion
\begin{align}\label{eq:rf_bellman}
J^\pi_T(z)=c\left(h_\eta(z)\right)+\gamma \mathbb{E}_{\tilde{z}\sim P_\pi(\cdot \mid z)}\left[J^\pi_{T-1}(\tilde{z})\right],
\end{align}
where we assume that $J^\pi_0(z)\equiv 0$.

We then prove by induction that $J^\pi_T$ is $L_T$-Lipschitz with respect to $d_\pi(z,z')$, with $L_T \triangleq 8{H}_wT$, which implies  \eqref{eq:rf_value_bound}.

For $T=0$ it holds trivially. Assume it holds for $T-1$. Then for any $z,z'$,
\begin{align*}
&\left|J^\pi_T(z)-J^\pi_T(z')\right| \leq \left|c\left(h_\eta(z)\right)-c\left(h_\eta(z')\right)\right| + \gamma\left|\mathbb{E}_{\tilde{z}\sim P_\pi(\cdot\mid z)}J^\pi_{T-1}(\tilde{z}) - \mathbb{E}_{\tilde{z}'\sim P_\pi(\cdot\mid z')}J^\pi_{T-1}(\tilde{z}')\right|.
\end{align*}

Note that the first term is bounded by the quantity $4{H}_w\|h_\eta(z)-h_\eta(z')\|_2$. On the other hand, for the second term, since $J^\pi_{T-1}$ is $L_{T-1}$-Lipschitz with respect to $d_\pi(z,z')$, we can use Kantorovich-Rubinstein duality to bound $\left| \mathbb{E}_{\tilde{z}\sim P_\pi(\cdot\mid z)}J^\pi_{T-1}(\tilde{z}) - \mathbb{E}_{\tilde{z}'\sim P_\pi(\cdot\mid z')}J^\pi_{T-1}(\tilde{z}') \right|$ with the expression $L_{T-1} W_1 \left(P_\pi(\cdot \mid z),P_\pi(\cdot\mid z')\right),$ i.e., by using \eqref{eq:KR_lispzhitz}.

Using the definition of the bisimulation metric, with the reward term dropped, we have
$W_1\left(P_\pi(\cdot \mid z),P_\pi(\cdot\mid z')\right)=\gamma^{-1}d_\pi(z,z')$.
Therefore, we obtain
\begin{align*}
\left|J^\pi_T(z)-J^\pi_T(z')\right| &\leq 4{H}_w\|h_\eta(z)-h_\eta(z')\|_2 + L_{T-1} d_\pi(z,z').
\end{align*}

As we assume that $\|h_\eta(z)-h_\eta(z')\|_2\leq d_\pi(z,z')+\varepsilon$,  then
\begin{align}
\left|J^\pi_T(z)-J^\pi_T(z')\right| \leq \left(4{H}_w+L_{T-1}\right)d_\pi(z,z') + 4{H}_w\varepsilon.
\end{align}

Unrolling the recursion and assuming that $\epsilon$ is sufficiently small, i.e., $\epsilon \leq d_\pi(z,z')$, then $L_T \triangleq  8{H}_wT$, which implies the bound as stated in Theorem \ref{thm:rf_generalization}.

\end{document}